\documentclass[twoside, 11pt]{article}
\usepackage[pdftex,letterpaper=true,pagebackref=false]{hyperref} 
\hypersetup{
    plainpages=false,      
    pdfpagelabels=true,    
    bookmarks=true,         
    unicode=false,          
    pdftoolbar=true,       
    pdfmenubar=true,     
    pdffitwindow=false,    
    pdfstartview={FitH},  
    pdftitle={(Bandit) Convex Optimization with Biased Noisy Gradient Oracles},    % set
    pdfauthor={},    
    pdfnewwindow=true,     
    colorlinks=true,        
    linkcolor=blue,         
    citecolor=blue,       
    filecolor=magenta,
    urlcolor=cyan         
}
\usepackage{amsmath}
\usepackage{ntheorem}

\usepackage[english]{babel}
\usepackage{authblk}
\usepackage{times}
\usepackage[margin=1in]{geometry}
\usepackage{makecell}
\usepackage{etex}
\usepackage{amsfonts}
\usepackage{amssymb}
\usepackage{mathtools}
\usepackage{graphicx}
\usepackage{algorithm}
\usepackage{natbib}
\usepackage[noend]{algpseudocode}
\usepackage{url}
\usepackage{bm}
\usepackage{booktabs}
\usepackage{multirow}
\usepackage{colortbl}
\usepackage{hhline}
\usepackage{paralist}
\usepackage{xspace}
\usepackage{sidecap}

\usepackage[capitalize]{cleveref}
\newtheorem{ass}{Assumption}
\crefname{ass}{Assumption}{Assumptions}

\newtheorem{theorem}{Theorem}
\newtheorem{lemma}{Lemma} 
\newtheorem{proposition}{Proposition} 
\newtheorem{remark}{Remark}

\newtheorem{definition}{Definition}
\newcommand{\BlackBox}{\rule{1.5ex}{1.5ex}}  % end of proof
\newenvironment{proof}{\par\noindent{\bf Proof\ }}{\hfill\BlackBox\\[2mm]}

\usepackage{subfigure}
\usepackage{tikz}
\usepackage{pgfplots}
\usepackage{makecell}
\usetikzlibrary{positioning,arrows.meta,shapes,calc,patterns,fit,decorations}
\usepgfplotslibrary{fillbetween}
\pgfplotsset{compat=1.10}

\usetikzlibrary{external}
\makeatletter
\def\BState{\State\hskip-\ALG@thistlm}
\makeatother

\newcommand{\cF}{\mathcal{F}}
\newcommand{\cS}{\mathfrak{F}}
\newcommand{\cA}{\mathcal{A}}
\newcommand{\cK}{\mathcal{K}}

\newcommand{\R}{\mathbb{R}}
\newcommand{\EE}[1]{\mathbb{E}\left[#1\right]}
\newcommand{\ip}[1]{\langle #1 \rangle}
\newcommand{\norm}[1]{\left\|#1\right\|}
\newcommand{\scnorm}[1]{\left\|#1\right\|}

\newcommand{\largenorm}[1]{\large\|#1\large\|}

\newcommand{\bignorm}[1]{\big\|#1\big\|}

\DeclareMathOperator{\fspan}{span}

\DeclareMathOperator{\argmin}{argmin}

\DeclareMathOperator{\dom}{dom}

\DeclareMathOperator{\essup}{ess\,sup}
\newcommand\numberthis{\addtocounter{equation}{1}\tag{\theequation}}
\newcommand{\DR}{D_{\mathcal{R}}}

\newcommand{\og}{\overline{\gamma}}

\newcommand{\C}{\mathcal{C}}
\newcommand{\A}{\mathcal{A}}

\newcommand{\B}{\mathcal{B}}

\newcommand{\F}{\mathcal{F}}
\newcommand{\N}{\mathcal{N}}
\newcommand{\E}{\mathbb{E}}
\renewcommand{\P}{\mathbb{P}}

\renewcommand{\O}{\mathcal{O}}
\newcommand{\D}{\mathcal{D}}

\newcommand{\K}{\mathcal{K}}
\newcommand{\cR}{\mathcal{R}}
\newcommand{\tf}{\tilde{f}}
\newcommand{\hp}{\hat{p}}
\newcommand{\hC}{\hat{C}}

\newcommand{\noise}{\psi}
\newcommand{\noiseCap}{\Psi}

\newcommand{\dnorm}[1]{\norm{#1}_*} 
\def\<{\left\langle} 
\def\>{\right\rangle}

\newcommand{\tvnorm}[1]{\norm{#1}_{\rm TV}}
\newcommand{\dkl}[2]{D_{\rm kl}\left({#1} |\!| {#2} \right)}
\newcommand{\normal}{\mathsf{N}}  
 
\newcommand{\indic}[1]{\mathbb{I}\left\{#1\right\}} 
\newcommand{\tr}{^\mathsf{\scriptscriptstyle T}}

\title{(Bandit) Convex Optimization with Biased Noisy Gradient Oracles\thanks{An earlier version of this paper was published at AISTATS 2016 \citep{HuPrGySz16}.}}
\author{Xiaowei  Hu$^1$,  Prashanth L.A.$^2$,  Andr\'as Gy\"orgy$^3$, and Csaba Szepesv\'ari$^1$}
\date{$^1$ Department of Computing Science, University of Alberta \\
	$^2$ Department of Computer Science and Engineering, 	Indian Institute of Technology Madras \\
	$^3$ Department of Electrical and Electronic Engineering, Imperial College London}

\begin{document}

\maketitle

%%%%%%%%%%%%%%%%%%%%%%%%%%%%%%%%%%%%%%%%%%%%%%%%%%%%%%%%%%%%%%
%%%%%%%%%%%%%%%%%%%%%%%%%%%%%%%%%%%%%%%%%%%%%%%%%%%%%%%%%%%%%%
%%%%%%%%%%%%%%%%%%%%%%%%%%%%%%%%%%%%%%%%%%%%%%%%%%%%%%%%%%%%%%
%%%%%%%%%%%%%%%%%%%%%%%%%%%%%%%%%%%%%%%%%%%%%%%%%%%%%%%%%%%%%%
%%%%%%%%%%%%%%%%%%%%%%%%%%%%%%%%%%%%%%%%%%%%%%%%%%%%%%%%%%%%%%
\begin{abstract}
Algorithms for bandit convex optimization and online learning often rely on constructing noisy gradient estimates, which are then used in appropriately adjusted first-order algorithms, replacing actual gradients.  Depending on the properties of the function to be optimized and the nature of ``noise'' in the bandit feedback, the bias and variance of gradient estimates exhibit various tradeoffs.  In this paper we propose a novel framework that replaces the specific gradient estimation methods with an abstract oracle.  With the help of the new framework we unify previous works, reproducing  their results in a clean and concise fashion, while, perhaps more importantly, the framework also allows us to formally show that to achieve the optimal root-$n$ rate either the algorithms that use existing gradient estimators, or the proof techniques used to analyze them have to go beyond what exists today.
\end{abstract}

\section{Introduction}
\label{sec:intro}
We consider convex optimization over a nonempty, closed convex set $\K \subset \R^d$:
\begin{align}
\mbox{Find } x^* = \arg\min_{x \in \K} f(x). \label{eq:pb}
\end{align}
Specifically, we operate in the \textit{stochastic bandit convex optimization} (BCO) setting where $f$ is a smooth convex function and an algorithm observes only noisy point evaluations of $f$. The goal is to find a near minimizer $\hat X_n \in \K$ of the objective $f$, after observing $n$ evaluations of $f$ at different points, chosen by the algorithm in a sequential fashion.
We also consider the \textit{online BCO} setting, where the environment chooses a sequence of smooth convex loss functions $f_1,\dots,f_n$ and the aim is to minimize the regret, which roughly translates to ensuring
 that the sum of the function values at the points chosen by the algorithm is not too far from the optimal value of $f = f_1+\dots+f_n$  (see \cref{sec:problem} for a detailed description).

Convex optimization is widely studied under different models concerning what can be observed about the objective function. 
These models range from accessing the full gradient to only observing noisy samples from the objective function (cf. \citealp{nesterov2004introductory,DeGliNe14,HaLe14:SOC,PoTsy90,flaxman2005online,AbHaRa08,AgDeXi10,Ne11:TR,AgFoHsuKaRa13:SIAM,katkul,kushcla,spall1992multivariate,spall1997one,Dip03:AoS,bhatnagar-book,duchi2015optimal}). 

In this paper, 
we present and analyze a novel framework that applies to both stochastic and online BCO settings. 
In our framework, an optimization algorithm can query an oracle repeatedly to get a noisy and biased version of the gradient of $f$ (or a subgradient for non-differentiable functions), where the algorithm querying the oracle also sets a parameter that controls the tradeoff between the bias and the variance of the gradient estimate.
This oracle model subsumes the majority of the previous works in the literature that are based on first-order methods for stochastic as well as BCO settings.

Gradient oracles have been considered in the literature before: Several previous works assume that the accuracy requirements hold with probability one \citep{dAsp08,Baes09,DeGliNe14} or consider adversarial noise \citep{SchRoBa11}. Gradient oracles with stochastic noise, which is central to our development, were also considered \citep{JN11a,Hon12,DvoGa15}; however, these papers assume that the bias and the variance are controlled separately, and consider the performance of special algorithms (in some cases in special setups).

The main feature of our model is that we allow stochastic noise, control of the bias and the variance. Our gradient oracle model applies to several gradient estimation techniques extensively used in the literature, mostly for the case when the gradient is estimated only based on noisy observations of the objective function \citep{katkul,kushcla,spall1992multivariate,spall1997one,Dip03:AoS,bhatnagar-book,duchi2015optimal}. A particularly interesting application of our model is the widely studied bandit convex optimization problem,
mentioned above,
where most previous algorithms 
essentially use gradient estimates and first-order methods 
\citep{PoTsy90,flaxman2005online,AbHaRa08,AgDeXi10,Ne11:TR,AgFoHsuKaRa13:SIAM,HaLe14:SOC}.

We study the achievable rates for stochastic and online BCO settings under our biased gradient oracle model.  
In particular, we provide both upper and lower bounds on the minimax optimization error for stochastic BCO (and regret for online BCO) for several oracle models, which correspond to different ways of quantifying the bias-variance tradeoff of the gradient estimate. From the results, we observe that the upper and lower bounds match for the case of smooth, convex functions, while there exists a gap for smooth, strongly convex functions. 
We do not claim to invent methods for proving upper bounds, as the methods we use have been known previously for special cases  (see the references above),
but our main contribution lies in abstracting away the properties of gradient estimation procedures, 
thereby unifying previous analysis, providing a concise summary and an explanation of differences between previous works.
More importantly, our framework also allows to prove lower bounds for any algorithm that relies on gradient
estimation oracles of the type our framework captures
(earlier work of \citealp{Chen88:LB-AoS} considered a related lower bound on the convergence of the iterate instead of the function value).
Of special interest may be a general method that we introduce and which allows us to reduce the problem of proving lower bounds of specific $d$-dimensional settings to their one-dimensional counterpart, thus allowing one to focus on the one-dimensional setting when proving lower bounds.

Note that our oracle model does not capture the full strength of the gradient estimates used in previous work, but it fully describes the properties of the estimates that \emph{so far have been used in their analysis}.
As a consequence, our lower bounds show that the known minimax regret of $\sqrt{T}$ \citep{BubeckDKP15,BuEl15,shamir2012complexity}
of online and stochastic BCO
cannot be shown to hold
for any algorithm that uses current gradient estimation procedures, unless the proof exploited finer properties
of the gradient estimators than used in prior works. In particular,
our lower bounds even invalidate the claimed (weaker) upper bound of \citet{DeElKo15}.

The rest of the paper is organized  as follows: We introduce the biased gradient oracle model in \cref{sec:problem}, provide the upper and lower bounds in \cref{sec:results} and describe applications to online and stochastic BCO in Sections~\ref{sec:sbco} and~\ref{sec:obco}, respectively. In \cref{sec:related}, we discuss related work and present the proofs in Sections \ref{sec:appendix-md}--\ref{sec:appendix-grad}. Finally, we provide concluding remarks in \cref{sec:conc}.

\section{Problem Setup}
\label{sec:problem}
\textit{Notation:} Capital letters will denote random variables.
For $i\le j$ positive integers,
 we use the notation $a_{i:j}$ to denote
 the sequence $(a_i,a_{i+1}, \dots, a_{j})$.
 We let $\| \cdot \|$ denote some norm on $\R^d$, whose dual is denoted by $\dnorm{\cdot}$. 
 Let $\K \subset \R^d$ be a convex body, i.e., a nonempty closed convex  set with a non-empty interior.
 Given the function $f:\K \to \R$ which is differentiable in $\K^\circ$,%
  \footnote{For $A\subset \mathbb{R}^d$, $A^\circ$ denotes the interior of $A$.}
 $f$ is said to be $\mu$-strongly convex w.r.t.\  $\| \cdot \|$  ($\mu\ge 0$) if
 $\tfrac{\mu}{2} \|x-y\|^2 \le \mathcal{D}_f(x,y)\doteq f(x)-f(y)-\ip{\nabla f(y),x-y}$, for all $x \in \K \,, y \in \K^\circ$.
Similarly, $f$ is $\mu$-strongly convex w.r.t.\  a \emph{function} $\cR$
	if $\tfrac{\mu}{2}\DR(x,y)\le\mathcal{D}_f(x,y)$ for all $x \in \K \,, y \in \K^\circ$, where $\cK^\circ\subseteq \dom(\mathcal{R})$ and $\mathcal{R}$ is differentiable over $\K^\circ$.
 A function $f$ is $L$-smooth w.r.t.\  $\| \cdot \| $ for some $L>0$ if
$D_f(x,y) \le \tfrac{L}{2} \|x-y\|^2$, for all $x \in \K \,, y \in \K^\circ$.
 This latter condition is equivalent to that $\nabla f$ is $L$-Lipschitz, that is, 
 $\norm{\nabla f(x) - \nabla f(y)}_* \le L \norm{x-y}$ \citep[Theorem~2.1.5]{nesterov2004introductory}.
 We let $\F_{L,\mu, \cR}(\K)$ denote the class of functions that are $\mu$-strongly convex w.r.t. $\cR$ and $L$-smooth w.r.t. some norm $\norm{\cdot}$ on the set $\K$ (typically, we will assume that $\cR$ is also strongly convex w.r.t. $\norm{\cdot}$).
 Note that $\F_{L,\mu,\cR}(\K)$ includes functions whose domain is larger than or equal to $\K$.
We also let $\F_{L,\mu}(\K)$ be $\F_{L,\mu,\cR}(\cK)$ with $\cR(\cdot) = \frac12 \norm{\cdot}_2^2$.
Then, the set of convex and $L$-smooth functions with domain including $\K$ is  $\F_{L,0}(\K)$.
Besides the standard big-$O(\cdot)$ notation, we will also use $\tilde{O}(\cdot)$:
For a positive valued function $f: \mathbb{N} \to \R_+$, $\tilde{O}(f)$ contains any $g:\N \to \R_+$ such that
$g =O( \log^p(n) f(n) )$ for some $p>0$.
(As usual, we abuse notation by writing $g= O(f)$ instead of $g\in O(f)$.)
Finally, we will denote the indicator function of an event $E$ by $\indic{E}$, that is $\indic{E}=1$ if $E$ holds and equals zero otherwise.
\begin{figure}
\begin{center}
\includegraphics[width=0.8\textwidth]{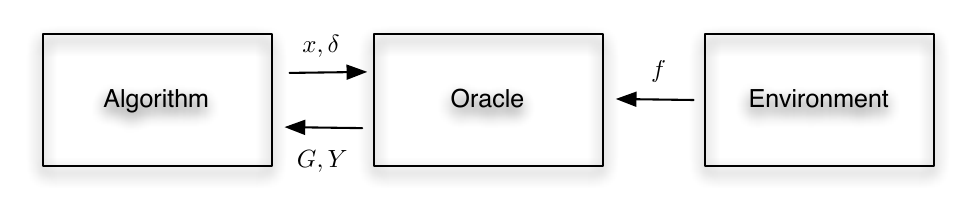}
\end{center}
\caption{The interaction of the algorithms with the gradient estimation oracle and the environment. For more information, see the text. \vspace{-0.3cm}}
\label{fig:oracle}
\end{figure}

In this paper, we consider convex optimization in a novel setting, both for stochastic as well as online BCO.
In the online BCO setting, the environment chooses a sequence of loss functions $f_1,\dots,f_n$ belonging to a set $\cF$ of convex functions over a common domain $\cK\subset \R^d$ which is assumed to be a convex body.
In the stochastic BCO setting, a single fixed loss function $f\in \cF$ is chosen.
An algorithm chooses a sequence of points $X_1,\dots,X_n\in \cK$ in a serial fashion.
The novelty of our setting is that the algorithm, upon selecting point $X_t$, receives
a noisy and potentially biased estimate $G_t\in \R^d$
of the gradient of the loss function $f$
(more generally, an estimate of a subgradient of $f$, in case $f$ is not differentiable at $X_t$).
To control the bias and the variance, the algorithm can choose a \emph{tolerance parameter} $\delta_t>0$
(in particular, we allow the algorithms to choose the tolerance parameter sequentially).
A smaller $\delta_t$ results in a smaller ``bias'' (for the precise meaning of bias, we will consider two definitions below), while typically with a smaller $\delta_t$, the ``variance'' of the gradient estimate increases.
Notice that in the online BCO setting, the algorithm suffers the loss $f_t(Y_t)$ in round $t$, where $Y_t\in \cK$\footnote{For simplicity, in some cases we allow $f$ to be defined outside of $\K$ and allow $Y_t$ to be in a small vicinity of $\K$.} is guaranteed to be in the $\delta_t$-vicinity of $X_t$.
The goal in the online BCO setting is to minimize the expected \emph{regret}, defined as follows:
	$$R_n =\EE{ \sum_{t=1}^n f_t(Y_t)} - \inf_{x\in \cK} \sum_{t=1}^n f_t(x).$$
In the stochastic BCO setting, the algorithm is also required to select a point $\hat{X}_n\in \cK$ once
the $n$th round is over (in both settings, $n$ is given to the algorithms)
and the algorithm's performance is quantified using the \emph{optimization error}, defined as
$$\Delta_n = \EE{f(\hat{X}_n)} - \inf_{x\in \cK} f(x).$$

The main novelty of the model is that the information flow between the algorithm and the environment (holding $f$, or $f_{1:n}$) is mediated by a stochastic gradient estimation oracle. As we shall see, numerous existing approaches to online learning and optimization based on noisy pointwise information fit in this framework.

We will use two classes of oracles. In both cases, the oracles are specified by
two functions $c_1,c_2:[0,\infty)\to [0,\infty)$, which will be assumed to be continuous,
monotonously increasing (resp., decreasing) with
$\lim_{\delta\to  0} c_1(\delta)=0$ and $\lim_{\delta\to 0} c_2(\delta)=+\infty$.
Typical choices for $c_1,c_2$ are $c_1(\delta) = C_1 \delta^p$, $c_2(\delta) = C_2\delta^{-q}$ with $p,q>0$.
Our type-I oracles are defined as follows:
\begin{definition}[$(c_1,c_2)$ type-I  oracle]
\label{def:oracle1}
We say that $\gamma$ is a  $(c_1,c_2)$ type-I oracle for $\cF$, if for any function $f\in \cF$,
$x\in \cK,0<\delta\le 1$, $\gamma$ returns random elements  $G\in \R^d$ and  $Y\in \cK$ such that
$\norm{x-Y}\le \delta$ almost surely (a.s.) and the following hold:
\begin{enumerate}
\item $\norm{ \EE{G}  - \nabla f(x)  }_* \le c_1(\delta) $ (bias); and
\item $\EE{\norm{ G -  \EE{G} }_*^2} \le c_2(\delta)$ (variance).
\end{enumerate}
\end{definition}
The upper bound on $\delta$ is arbitrary: by changing the norm, any other value can also be accommodated. Also, the upper bound only matters when $\K$ is bounded and the functions in $\cF$ are defined only in a small vicinity of $\K$.

The second type of oracles considered is as follows: 
\begin{definition}[$(c_1,c_2)$ type-II  oracle]
\label{def:oracle2}
We say that $\gamma$ is a  $(c_1,c_2)$ type-II oracle for $\cF$, if for any function $f\in \cF$,
$x\in \cK,0<\delta\le 1$, $\gamma$ returns $G\in \R^d$ and  $Y\in \cK$ random elements such that $\norm{x-Y}\le \delta$ a.s. and the following hold:
\begin{enumerate}
\item There exists $\tilde{f} \in \cF$ such that
$\bignorm{\tilde{f}- f}_\infty \le c_1(\delta)$  and
$\EE{G}  = \nabla \tilde{f}(x)$ (bias); and
\item $\EE{\norm{ G -  \EE{G} }_*^2} \le c_2(\delta)$ (variance).
\end{enumerate}
\end{definition}
We will denote the set of type-I (type-II) oracles satisfying the $(c_1,c_2)$-requirements given a function $f\in \cF$ by $\Gamma_1(f,c_1,c_2)$ (resp., $\Gamma_2(f,c_1,c_2)$).

Note that while a type-I oracle returns a biased, noisy gradient estimate for $f$, 
a type-II oracle returns an unbiased, noisy gradient estimate for some function $\tilde{f}$ which is close to $f$.
Note that $\tilde{f}$ is allowed to change with the inputs (not only by $f$, but also with $x$ and $\delta$) in the definition.
In fact, the oracles (in both cases)  can have a memory of previous queries and depending on the memory
can respond to the same inputs $(x,\delta,f)$ with a differently constructed pair.%
\footnote{For oracles with memory, in the definition (and in the proofs provided later in the paper) the expectation should be replaced with
an expectation that is conditioned on the past.}
The oracles that we use will nevertheless be memoryless.

As noted above, even a memoryless type-II oracle can respond such that $\tilde{f}$ depends on $x$ or $\delta$.
A type-II oracle is called a \emph{uniform} type-II oracle if $\tilde{f}$ only depends on $f$ (and possibly the history of previous queries), but not on $x$ and $\delta$.
The type-II oracles that will be explicitly constructed will be all uniform.

We call an oracle (type-I or II)  \emph{unbiased} if $\EE{Y}=x$ in the above definitions. Note that if the oracle is unbiased and the loss function is smooth,
 an algorithm does not loose ``too much'' from suffering loss at $Y$ instead of the query point $x$
 since in this case we have 
$$\EE{f(Y)}-f(x) \le \EE{\ip{\nabla f(x), Y-x}+\tfrac{L}{2}\norm{Y-x}^2} \le L\delta^2/2.$$

Examples of specific oracle constructions (based on previous works of others) will be given in  \cref{sec:sbco}. We also note that for type-II oracles we only need properties of the function class which the surrogate function $\tilde{f}$ belongs to; the assumption $f \in \F$ is only included to simplify the definition (e.g., some oracles work for non-convex functions $f$ for which a suitable convex surrogate and the associated oracle exists).

As the next result shows, type-I and II oracles are closely related.
In particular, a type-I oracle is also a type-II oracle (although not a uniform type-II oracle). On the other hand, type-II oracles need to satisfy an alternative condition to become type-I oracles as the closeness of $\tilde{f}$ and $f$ is insufficient to conclude anything about the distance of their gradients:
\begin{proposition}\label{thm:typered}
\cref{def:oracle1} is a sufficient condition for \cref{def:oracle2}, given a bounded $\cK$. 
In particular, letting $R = \sup_{y \in \cK}\norm{y}$, 
for any $f,c_1,c_2$ such that $f+\ip{c,\cdot} \in \F$ for any $\norm{c}_* \le c_1(1)$, it holds that $\Gamma_1(f,c_1,c_2) \subset \Gamma_2(f,Rc_1,c_2)$. 
Furthermore, if $\bignorm{\tilde{f}-f}_\infty \le c_1(\delta)$ is replaced by 
\begin{align}
\label{eq:oracle2alt}
%\sup_{x\in \cK}
\bignorm{\nabla \tilde{f}- \nabla f}_* \le c_1(\delta)
\end{align}
in \cref{def:oracle2} (for all $x\in\cK$ and $0<\delta \le 1$), then any oracle satisfying this modified definition  is also a $(c_1,c_2)$ type-I oracle.
\end{proposition}

\begin{proof}
We prove only the first part of the claim, as the second part follows by Definitions~\ref{def:oracle1}--\ref{def:oracle2} and condition \eqref{eq:oracle2alt}.
 %is immediate from the definitions, hence it remains to prove the first part.
Let $\gamma$ be  a $(c_1,c_2)$ type-I oracle. Fix $x$, $\delta,f$ and let the oracle's response be $G,Y$. 
Define $\tilde{f}:\cK \to \mathbb{R}$ by
\[\tilde{f}(y) =\EE{ f(y)+ \ip{G-\nabla f(x) , y}},\]
where the expectation is over the randomness of $G$ (note that $\tilde{f}$ depends on $x$ and $\delta$).
Then, $\nabla \tilde{f}(y) =  \nabla f(y)-\nabla f(x) + \EE{G}$
and thus substituting $x$ for $y$ we get that $\nabla \tilde{f}(x) = \EE{G}$.
Further, 
using $\norm{ \EE{G}  - \nabla f(x)  }_* \le c_1(\delta) $,
we have, for any $y \in \cK$,
\begin{align*}
|\tilde{f}(y)-f(y)|
=
\left| \EE{\ip{G-\nabla f(x), y}}\right|
 \le \largenorm{\EE{G}-\nabla f(x)}_* \, \norm{y}
 \le  R\,c_1(\delta)\,.
\end{align*}
From the above, it follows that $\gamma$ is also an $(Rc_1,c_2)$ Type-II oracle, since $\tilde{f}\in\F$ by the conditions of the proposition. 
\end{proof}

In the online convex optimization setting,
algorithms are compared based on their minimax \emph{regret}, 
whereas in the stochastic convex optimization setting, the algorithms are compared based on their
 minimax \emph{error}
 (sometimes, also called as the ``simple regret'').
Both regret notions are defined with respect to a class of loss functions $\cF$, and the bias/variance control functions $c_1,c_2$.
The \emph{worst-case regret} of algorithm $\A$ interacting with $(c_1,c_2)$ type-I oracles for the function class $\F$ is
defined as
\begin{align*}
%\MoveEqLeft
R_{\F,n}^\cA(c_1,c_2)
&=  \sup_{f_{1:n}\in \cF^n}
	\sup_{\substack{\gamma_t \in \Gamma_1(f_t,c_1,c_2)\\1\le t \le n
	}} R_n^{\cA}(f_{1:n},\gamma_{1:n})
\end{align*}
where $R_n^{\cA}(f_{1:n},\gamma_{1:n})$ denotes the expected regret of $\cA$ (against $f_{1:n},\gamma_{1:n}$), and
the \emph{minimax expected regret} for $(\cF,c_1,c_2)$ with type-I oracles is defined as
\[
R_{\F,n}^*(c_1,c_2) = \inf_{\cA} R_{\F,n}^\cA(c_1,c_2),
\]
where $\cA$ ranges through all algorithms that interact with the loss sequence  $f_{1:n}= (f_1,\dots,f_n)$
through the oracles $\gamma_{1:n}$ (in round $t$, oracle $\gamma_t$ is used).
The minimax regret for type-II oracles is defined analogously.

In the stochastic BCO setting, the \emph{worst-case error} is defined through
\begin{align}
\label{eq:minimaxerrdef}
\Delta_{\F,n}^\cA(c_1,c_2)
= \sup_{f \in \cF} \sup_{\gamma\in \Gamma_1(f,c_1,c_2)}  \Delta_n^{\cA}(f,\gamma)\,, %(f,\gamma)\,,
% \EE{ f(\hat{X}_n) } - \inf_{x\in \cK}  f(x)\,,
\end{align}
where $\Delta_n^{\cA}(f,\gamma)$ is the optimization error that $\cA$ suffers
after $n$ rounds of interaction with $f$ through (a single) $\gamma$ as described earlier, and the \emph{minimax error}
is defined as
\[
\Delta_{\F,n}^*(c_1,c_2) =  \inf_{\cA} \Delta_{\F,n}^\cA(c_1,c_2),
\]
where, again, $\cA$ ranges through all algorithms that interact with $f$ through an oracle.
The minimax error for type-II oracles is defined analogously.

Consider now the case when the set $\K$ is bounded and, in particular, assume that 
$\K$ is included in the unit ball w.r.t. $\norm{\cdot}$.
Assume further that the function set $\F$ is invariant to linear shifts
(that is for any $f\in \F$, $w\in \R^d$, $x\mapsto f(x) + \ip{x,w}$ is also in $\F$).
Let
 $\Delta_{n}^{\mathrm{type-I}}$ and $\Delta_{n}^{\mathrm{type-II}}$ denote  the appropriate minimax errors for the two types of oracles.
Then, by the construction in  \cref{thm:typered},
\begin{align}
\Delta_{\F,n}^{\mathrm{type-I}}(c_1,c_2) \le \Delta_{\F,n}^{\mathrm{type-II}}(Rc_1,c_2)\,.
\end{align}
%As a result, when proving lower bounds, we shall consider type-I, while when proving upper bounds we will consider type-II oracles.
Note that $R$ may depend on the dimension $d$, e.g., for $\mathcal{K} = \left[ -1,1 \right]^d$, $R = \sqrt{d}$ when using the Euclidean norm. To clarify the different $c_1$ used by type-I and II oracles, we will present the upper and lower bounds separately for the two oracle types, although the type-I upper bound can actually be derived from type-II (and the type-II lower bound can be derived from type-I). 
Also note that for either type of oracles, $\Delta_{\F,n}^*(c_1,c_2) \le R_{\F,n}^*(c_1,c_2)/n$. This follows by the well known construction that turns an online convex optimization method $\A$ for regret minimization into an optimization method by running the method and at the end choosing $\hat{X}_n$ as the average of the points $X_1,\dots,X_n$ queried by $\A$ during the $n$ rounds.
Indeed, then $f(\hat{X}_n) \le \frac1n \sum_{t=1}^n f(X_t)$ by Jensen's inequality, hence the average regret of $\A$ will upper bound the error of choosing $\hat{X}_n$ at the end.
A consequence of this relation is that a lower bound for $\Delta_{\F,n}^*(c_1,c_2) $ will also be a lower bound for $R^*_{\F,n}(c_1,c_2)/n$ and an upper bound on $R^*_{\F,n}(c_1,c_2)$ leads to an upper bound on $\Delta_{\F,n}^*(c_1,c_2)$. This explains why we allowed taking supremum over time-varying oracles in the definition of the regret and why we used a static oracle for the optimization error: to maximize the strength of the bounds we obtain.

\section{Main Results}
\label{sec:results}
In this section we provide our main results in forms of upper and lower bounds on the minimax error.
First we give an upper bound for the mirror-descent algorithm shown as \cref{alg}.
In the algorithm, we assume that the regularizer function $\mathcal{R}$ is $\alpha$-strongly convex and the target function $f$ is smooth or  smooth and strongly convex.
We give results for polynomial oracles, that is, when $c_1$ and $c_2$ are polynomial functions (in particular, monomial functions)
of their argument. The reason, as we will see, is that existing oracle constructions give rise to polynomial oracles for the function classes that we consider.

\begin{algorithm}[t]
\begin{algorithmic}
    \State {\bf Input:}  Closed convex set $\cK\ne \emptyset$, regularization function $\mathcal{R}:\dom(\mathcal{R})\to \mathbb{R}$, $\cK^{\circ}\subset \dom(\mathcal{R})$, tolerance parameter $\delta$, learning rates $\{\eta_t\}_{t=1}^{n-1}$.
%     In round $t=1, 2, \cdots, n-1$:
\State Initialize $X_1\in \cK$ arbitrarily.
\For{$t=1, 2, \cdots, n-1$}
	\State Query the oracle at $X_t$ to receive $G_t$, $Y_t$.
	\State Set
	$X_{t+1}=\argmin_{x\in \mathcal{K}}\left[ \eta_{t} \ip{G_t,x}+D_{\mathcal{R}}(x,X_t) \right].$
\EndFor
\State {\bf Return:} $\hat{X}_n = \frac{1}{n}\sum_{t=1}^n X_t \,.$
\end{algorithmic}
\caption{Mirror Descent with Type-I/II Oracle.
\label{alg}}
\end{algorithm}

\begin{theorem}[\textit{Upper bound}]
\label{thm:ub}
Consider the class $\F=\F_{L,0}$ of convex, $L$-smooth functions whose domain includes the bounded, convex body $\cK \subset \R^d$. 
Assume that the regularization function $\mathcal{R}$ is $\alpha$-strongly convex with respect to (w.r.t.)  some norm $\norm{\cdot}$, and $\cK^\circ\subseteq \dom(\mathcal{R})$.
For any $(c_1,c_2)$ type-I or any memoryless uniform $(c_1,c_2)$ type-II oracle 
with $c_1(\delta) = C_1 \delta^p$, $c_2(\delta) = C_2 \delta^{-q}$, $p,q>0$, 
the worst-case error (and hence the minimax error) of \cref{alg} run with an appropriate parameter setting
can be bounded as 
 \begin{align}
 \label{eq:MDbound1TypeI}
 \Delta_{\F_{L,0},n}^{MD, \mathrm{type-I}}(c_1,c_2) 
 &\leq  K_1 D^{ \frac{1}{2}} C_1^{\frac{q}{2p+q} } C_2^{\frac{p}{2p+q}} n ^{-\frac{p}{2p+q}} + \frac{1}{n}\left(\EE{f(X_1)-\inf_{x \in \cK}f(x)}+\frac{DL}{\alpha}  \right),
 \end{align}
 \begin{align}
 \label{eq:MDbound1}
 \Delta_{\F_{L,0},n}^{MD, \mathrm{type-II}}(c_1,c_2) 
 &\leq  K'_1 D^{ \frac{p}{2p+q}} C_1^{\frac{q}{2p+q} } C_2^{\frac{p}{2p+q}} n ^{-\frac{p}{2p+q}} + \frac{1}{n}\left(\EE{f(X_1)-\inf_{x \in \cK}f(x)}+\frac{DL}{\alpha}  \right),
%\vspace{-0.3cm}
 \end{align}
 where $D=\sup_{x,y\in \cK} \DR(x,y)$.

For the class $\F=\F_{L,\mu,\cR}$ of $\mu$-strongly convex (w.r.t.\  $\cR$) and $L$-smooth functions, with \\$\alpha >2L/\mu$, 
we have
\begin{align}
\Delta_{\F_{L,\mu,\cR},n}^{MD, \mathrm{type-I}}(c_1,c_2) 
&\leq  K_2D^{\frac{q}{2(p+q)}}C_1^{\frac{q}{p+q}}C_2^{\frac{p}{p+q}} \left( \frac{\log n+1+\frac{\alpha \mu}{\alpha \mu -2L}}{n} \right)^{\frac{p}{p+q}} \nonumber\\
&\qquad + \frac{1}{n}\EE{f(X_1)-\inf_{x \in \cK}f(x)} 
\,.
\label{eq:MDbound2TypeI}\\
\Delta_{\F_{L,\mu,\cR},n}^{MD, \mathrm{type-II}}(c_1,c_2) 
& 
\leq  K'_2C_1^{\frac{q}{p+q}}C_2^{\frac{p}{p+q}} \left( \frac{\log n+1+\frac{\alpha \mu}{\alpha \mu -2L}}{n} \right)^{\frac{p}{p+q}} \nonumber\\
&\qquad + \frac{1}{n}\EE{f(X_1)-\inf_{x \in \cK}f(x)} 
\,.
\label{eq:MDbound2}
\end{align}
In the above,
the constants $K_1$, $K'_1$, $K_2$ and $K'_2$ depend on $p, q, \alpha, \mu$.%
\footnote{
In particular,
$K_1 = 2^{\frac{q}{2(2p+q)}} \left( \alpha^{-1}+2\alpha^{-\frac{q}{2(p+q)}} \right) \left( \frac{2p+q}{2p} \right)^{\frac{p}{2p+q}}$,
$K'_1 = 3 \left(2+\frac{2}{n}\right)^{\frac{q}{2p+q} } \alpha^{-\frac{p}{2p+q}}\left(\frac{2p+q}{2p} \right)^{\frac{p}{2p+q}}$,
$K_2=2^{\frac{q}{2(p+q)}}\alpha^{-\frac{2p+q}{2(p+q)}}\mu^{-\frac{p}{p+q}}$
and
$K'_2=2^{\frac{q}{p+q}}\alpha^{-\frac{p}{p+q}}\mu^{-\frac{p}{p+q}}$.}

If the oracle is unbiased (but may be non-uniform, and may have memory) and either
(i) the oracle is of type-I or (ii) the oracle is of type-II and all functions in $\F$ have bounded gradients\footnote{This follows from the smoothness if, for example, the functions in $f$  are bounded.} then, for $\F \subset \F_{L,0}$, the regret of \cref{alg} run with an appropriate parameter setting can be bounded as
\[
\frac{1}{n} R_{\F}^{MD}(c_1,c_2) = O\left( D^{\frac{\hp}{2\hp+q}} \hC_1^{\frac{q}{2\hp+q}} C_2^{\frac{\hp}{2\hp+q}}  n^{-\frac{\hp}{2\hp+q}} \right)
\]
where $\hp=\min\{p,2\}$, $\hC_1=C_1 \indic{p \le 2} + (L/4) \indic{p \ge 2}$ for type-II oracles and $\hC_1=R C_1 \indic{p \le 2} + (L/4) \indic{p \ge 2}$ for type-I oracles where $R = \sup_{x\in\cK} \|x\|$.\footnote{The coefficient associated with the dominating term of the bound is
$2^{1+\frac{q/2}{2\hp+q}}(2\hp+q) (2\hp \alpha)^{-\frac{\hp}{2\hp+q}} q^{-\frac{q}{2\hp+q}}$.} In the strongly convex case, that is, when $\F \subset \F_{L,\mu}$, an appropriate parameter setting of \cref{alg} yields a regret bound\footnote{
The coefficient associated with the main term of the bound is $(\hp+q)\hp^{-\frac{\hp}{\hp+q}} q^{-\frac{q}{\hp+q}} (\alpha \mu)^{-\frac{\hp}{\hp+q}}$.}
\[
\frac{1}{n} R_{\F}^{MD}(c_1,c_2) = O\left(\hC_1^{\frac{q}{\hp+q}} C_2^{\frac{\hp}{\hp+q}} n^{-\frac{\hp}{\hp+q}} (1+\log n)^{\frac{\hp}{\hp+q}} \right).
\]
\end{theorem}
The proof of this theorem follows the steps of the standard analysis of the mirror descent algorithm
 and is provided in \cref{sec:appendix-md}, mainly for completeness 
 and because it is somewhat cumbersome to extract from the existing results what properties of the oracles they use.
 Comparing the bounds on the optimization error and the regret for the non-strongly convex case, note that $\hp$ plays the same role as $p$ and $\hC_1$ as $C_1$. The reason for the difference is that the extra loss introduced by using $Y_t$ instead of $X_t$ in the regret minimization case brings in an extra $L \delta^2/2$ term (as discussed at the introduction of unbiased oracles), and this term dominates the $C_1 \delta^p$ bias term when $p>2$, and increases its coefficient for $p=2$; $\hp$ and $\hC_1$ are obtained as the exponent and the coefficient of the dominating term from these two. On another note, the dependence on $D$ for type-I oracles seems different for the optimization and the regret minimization cases. However, by the strong convexity of $\cR$, $R \le \sqrt{2 D/\alpha}$ (when $\cR$ is also $L'$-smooth, $R \ge \sqrt{2 D/L'}$, so $R$ is of the same order as $\sqrt{D}$); applying this inequality gives the same dependence on $D$ for both types of oracles (for $p\ge 2$, the main term scales with a smaller power of $D$ for regret minimization due to the approximation issues discussed beforehand).
 
We next state lower bounds for both convex as well as strongly convex function classes. In particular, we observe that for convex and smooth functions the upper bound for the mirror descent scheme matches the lower bound, up to constants, whereas there is a gap for strongly convex and smooth functions.
Filling the gap is left for future work.
\begin{theorem}[\textit{Lower bound}]
\label{thm:lb-convex}
Let $n>0$ be an integer, $p,q>0$, $C_1,C_2>0$, 
$\cK\subset \R^d$ convex, closed, with  $[+1,-1]^d\subset \cK$.
Then, for any algorithm that observes $n$ random elements from a $(c_1,c_2)$ type-I oracle 
 with $c_1(\delta) = C_1 \delta^p$, $c_2(\delta) = C_2 \delta^{-q}$,
 the minimax error (and hence the regret) satisfies the following bounds:
 \begin{itemize}
 \item
${\F_{L,0}(\K)}$ (Convex and smooth) w.r.t. the Euclidean norm $\|\cdot\|_2$ with $L\ge \frac12$
\[
 \Delta_{\F_{L,0},n}^{*, \mathrm{type-I}}(c_1,c_2) \ge K_3\sqrt{d} \, C_1^{\frac{q}{2p+q}} C_2^{\frac{p}{2p+q}} n^{-\frac{p}{2p+q}}, 
\]
\[
 \Delta_{\F_{L,0},n}^{*, \mathrm{type-II}}(c_1,c_2) \ge K_3d^{\frac{p}{2p+q}} \, C_1^{\frac{q}{2p+q}} C_2^{\frac{p}{2p+q}} n^{-\frac{p}{2p+q}}, 
\]
\item
${\F_{L,1}(\K)}$ ($1$-strongly convex and smooth) with $L\ge 1$
\[
\Delta_{\F_{L,1}, n}^{*, \mathrm{type-I}}(c_1,c_2) \ge K_4 \,  C_1^{\frac{2q}{2p+q}} C_2^{\frac{2p}{2p+q}} n^{-\frac{2p}{2p+q}}\,. 
\]
\[
\Delta_{\F_{L,1}, n}^{*, \mathrm{type-II}}(c_1,c_2) \ge K_4 \,  D^{-\frac{q}{2p+q}}C_1^{\frac{2q}{2p+q}} C_2^{\frac{2p}{2p+q}} n^{-\frac{2p}{2p+q}}\,. 
\]
\end{itemize}
In the above,
the constants $K_1$ and $K_2$ depend on $p$ and $q$ only.%
\footnote{
In particular,
$K_3= \frac{\left(2p+q\right)^2}{2q^{\frac{q}{2p+q}}\left(4p+q\right)^{\frac{4p+q}{2p+q}}}$
and
$K_4= 2^{\frac{2p-q}{2p+q}} \frac{(2p+q)^3}{q^{\frac{2q}{2p+q}}(6p+q)^{\frac{6p+q}{2p+q}}}$.}
\end{theorem}
The proof of the above theorem is given in \cref{sec:appendix-lb-proof}. The proof involves using standard information-theoretic tools to first establish the result in one dimension and then a separability argument (see \cref{lemma:sep}) is employed to extend the result to $d$-dimensions. 
For the proof in one dimension, we provide a family of functions and a type-I oracle such that any algorithm suffers at least the stated error on one of the functions. In particular, for $\F_{L,0}$ with $L\ge 1/2$ we use
$f_{v,\epsilon}(x)=\epsilon\left( x-v\right)+2\epsilon^2 \ln\left(1+e^{-\frac{x-v}{\epsilon}}  \right)$ with $v=\pm 1$, $\epsilon>0$, and $x \in \cK \subset \R$ for appropriate $\epsilon$. Note that for \emph{any} $\epsilon>0$, $f_{v,\epsilon}\in \F_{1/2,0}\setminus \cup_{0<\lambda<1/2} \F_{\lambda,0}$.
To the best of our knowledge, the separability argument that we employ to relate the minimax error in $d$-dimensions to that in one dimension, is novel. 

\begin{remark}
By continuity, the result in \cref{thm:lb-convex} can be extended to cover the case of $q=0$ (constant variance). 
For the special case of $p=0$ and $C_1>0$, which implies a constant bias, it is possible to derive an $\Omega(1)$ lower bound by tweaking the proof. On the other hand, the case of $p=0$ and $C_1=0$ (no bias) leads to an $\Omega(d/\sqrt{n})$ lower bound. 
\end{remark}

\begin{remark}(\textbf{\textit{Scaling}})
For any function class $\F$, by the definition of the minimax error \eqref{eq:minimaxerrdef}, it is easy to see that
$$\Delta_n^*(\mu \F, c_1,c_2) = \mu \Delta_n^*\left(\F, c_1/\mu,c_2/\mu^2\right),$$
 where $\mu \F$ denotes the function class comprised of functions in $\F$, each scaled by $\mu>0$. In particular, this relation implies that the bound for $\mu$-strongly convex function class is only a constant factor away from the bound for $1$-strongly convex function class.
\end{remark}

\cref{tab:mse-1} presents the upper and lower bounds for two specific choices of $p$ and $q$ (relevant in applications, as we shall see later). These bounds can be inferred from the results in Theorems~\ref{thm:ub} and~\ref{thm:lb-convex}.

\begin{table*}
\small
\centering
 \begin{tabular}{|c|c|c|c|c|}
\hline
  \multirow{2}{*}{\textbf{Type-I Oracle}} & \multicolumn{2}{c}{\textbf{Convex + Smooth}} & \multicolumn{2}{|c|}{\textbf{Strongly Convex + Smooth}} \\[1em]
 \cline{2-5}
 & \textbf{Upper bound} & \textbf{Lower bound} & \textbf{Upper bound} & \textbf{Lower bound}\\
 \hline
\makecell{\\[-1.8em] \textbf{ $\delta$-bias, $\delta^{-2}$-variance} \\ ($p=1$, $q=2$) \\[0.2em]}& $\left(\dfrac{C_1^{2}C_2 D^2}{n}\right)^{1/4}$ & $\left(\dfrac{C_1^2 C_2 d^2}{n}\right)^{1/4}$ & $\left(\dfrac{C_1^2 C_2 D}{n}\right)^{1/3}$  & $\left(\dfrac{C_1^2 C_2}{ n}\right)^{1/2}$ \\
 \hline
%%%%%%%%%%%%%%%%%%%
\makecell{\\[-0.6em]\textbf{$\delta^2$-bias, $\delta^{-2}$-variance } \\[0.2em]  ($p=2$, $q=2$) \\[0.4em]}& $\left(\dfrac{C_1 C_2 \sqrt{D^3}}{n}\right)^{1/3}$  & $\left(\dfrac{C_1 C_2 \sqrt{d^3}}{n}\right)^{1/3}$ & $\left(\dfrac{C_1 C_2\sqrt{D}}{n}\right)^{1/2}$  & $\left(\dfrac{C_1 C_2 }{ n}\right)^{2/3}$\\
\hline
%%%%%%%%%%%%%%%%
\end{tabular}
\caption{Summary of upper and lower bounds on the minimax optimization error for different smooth function classes and  gradient oracles for the settings of \cref{thm:ub} and \cref{thm:lb-convex}. Note that when $\cR$ is the squared norm and $\K$ is the hypercube (as in the lower bounds), $D=\Theta(d)$ in the upper bounds and also that $C_1$, $C_2$ may hide dimension-dependent quantities for the common gradient estimators, as will be discussed later.
}
\label{tab:mse-1}
\end{table*}

\section{Applications to Stochastic BCO}
\label{sec:sbco}
The main application of the biased noisy gradient oracle based convex optimization of the previous section
is bandit convex optimization. We introduce here briefly the stochastic version of the problem, while online BCO will be considered in \cref{sec:obco}. Readers familiar with these problems and the associated
gradient estimation techniques, may skip this description to jump directly to \cref{thm:aaa},
and come back here only if clarifications are needed.

In the \emph{stochastic BCO} setting,
the algorithm sequentially chooses the points $X_1,\dots,X_n\in \cK$ while observing the loss function at these points in noise.
In particular, in round $t$, the algorithm chooses $X_t$ based on the earlier observations $Z_1,\dots,Z_{t-1}\in \R$ and $X_1,\dots,X_{t-1}$, after which it observes $Z_t$, where $Z_t$ is the value of $f(X_t)$ corrupted by ``noise''.

Previous research considered several possible constraints connecting $Z_t$ and $f(X_t)$.
One simple assumption is that $\{Z_t-f(X_t)\}_t$ is an $\{\cF_t\}_t = \{\sigma(X_{1:t},Z_{1:t-1})\}_t$-adapted martingale difference sequence (with favorable tail properties).
A specific case is when $Z_t - f(X_t) = \xi_t$, where $(\xi_t)$ is a sequence of independent and identically distributed (i.i.d.) variables.
A stronger assumption, common in stochastic programming, is that 
\begin{equation}
\label{eq:controlled}
Z_t = F(X_t,\noiseCap_t), \quad f(x) = \int F(x,\psi) P_{\noiseCap}(d\psi)\,,
\end{equation} 
where $\noiseCap_t\in \R$ is chosen by the algorithm and in particular the algorithm can draw $\noiseCap_t$ 
at random from $P_{\noiseCap}$.
As in \cite{duchi2015optimal}, we assume that the function $F(\cdot, \psi)$ is $L_{\psi}$-smooth $P_{\noiseCap}$-a.s. and the quantity $\overline{L}_{\Psi} = \sqrt{\E [L_{\Psi}^2]}$ is finite.  
Note that the algorithm is aware of $P_{\noiseCap}$, but does not know how different values of $\noise$ affect the noise $\xi(x,\noise)=F(x,\noise)-f(x)$. Nevertheless, as the algorithm can control $\noise$ and thus $\xi$, we refer to this as
\emph{controlled noise} setting and to the others as the case of \emph{uncontrolled noise}.
As we will see, and is well known in the simulation optimization literature \citep{KlSpNa99,duchi2015optimal},
this extra structure allows the algorithm to reduce the variance of the noise of its gradient estimates by reusing the same $\noiseCap_t$ in consecutive measurements, while measuring the gradient at the same point, an instance of the method of the method of common random variables.
As creating an estimate from $K$ points (which is equivalent to the so-called ``multi-point feedback setup'' from the literature where $K$ points are queried in each round) changes the number of rounds from $n$ to $n/K$, which does not change the convergence rate as long as $K$ is fixed. 

\subsection{Estimating the Gradient}

A common popular idea in bandit convex optimization is to use the bandit feedback to construct noisy (and biased) estimates of the gradient.
In the following, we provide a few examples for oracles that construct gradient estimates for function classes that are increasingly general: from smooth, convex to non-differentiable functions.

\paragraph{One-point feedback}
Given $x\in \cK$, $0<\delta\le 1$, common gradient estimates that are
based on a single query to the function evaluation oracle (the so-called
``one-point feedback'') take the form 
%\vspace{-0.2cm}
\begin{equation}
  \label{eq:one-point}
G = \frac{Z}{\delta}V, \textrm{ where } Z = f(x+\delta U) + \xi\,,
%\vspace{-0.2cm}
\end{equation}
where $(U,V,\xi)\in \R^d\times \R^d\times \R$ are jointly distributed random variables,
$\xi$ is the function evaluation noise whose distribution may depend on $x+\delta U$ but $\E[\xi|V]=0$, and $G$ is the estimate of $\nabla f(x)$ ($f:\cK\to \R$). 

In all oracle constructions we will use the following assumption:
\begin{ass}
  \label{ass:gradbasic}
  Let $\K \subset \D^\circ \subset \R^d$, where $f:\D \to \R$. %\footnote{Here, $\D^\circ$ denotes the interior of $\D$.}
  For any $x\in \K$, $x+\delta U \in \D$ a.s.,
  and $\EE{\norm{V}_*^2}$, $\EE{ \norm{U}^3 }<+\infty$.
\end{ass}
Note that here the function domain $\D$ can be larger than or equal to the set $\K$, where the algorithm chooses $x$. This is to ensure that the oracle will not receive invalid inputs, that is, queries where $f$ is not defined.
When the functions are defined over $\K$ only and $\K$ is bounded, the above constructions only work for $\delta$ small enough.
In this case, the best approach perhaps is to use Dikin ellipsoids to construct the oracles, as done by \citet{HaLe14:SOC}.

The next proposition, whose proof is based on ideas from \citet{spall1992multivariate}
 shows that the above one-point gradient estimator leads to a type-I (and, hence, also type-II) oracle.
\begin{proposition}
\label{prop:grad-onepoint}
Let \cref{ass:gradbasic} hold and let $\gamma$ be the one-point feedback oracle defined in \eqref{eq:one-point}.
Assume further that
  $U$ is symmetrically distributed,
  $V = h(U)$, where $h:\R^d \to \R^d$ is an odd function,
  $\EE{V}=0$, and $\E[V U\tr] = I$.
Then, in the uncontrolled noise case, $\gamma$ is a $(c_1(\delta),c_2(\delta))$ type-I oracle given in \cref{tab:oracles}, where 
$C_2=4 \EE{\norm{V}_*^2}\left( \essup \E[\xi^2|V]+\sup_{x\in\D} f^2(x)\right)$, and $C_1 =
\frac{L}{2} \E[ \dnorm{V} \norm{U}^2]$ when $f \in \F_{L,0}$ and  
$C_1 = \tfrac{B_3}{6} \EE{ \norm{V}_* \norm{U}^3 }$ for  $f \in \C^3$ where $B_3 = \sup_{x\in D} \norm{ \nabla^3 f(x) }_T$ where $\norm{\cdot}_T$ denotes the implied norm for rank-3 tensors.%
\end{proposition}

Another possibility is to use the so-called smoothing technique
\citep{PoTsy90,flaxman2005online,HaLe14:SOC}
to obtain type-II oracles. Following the analysis in \citet{flaxman2005online}, one gets the 
following result, which improves the bias of the previous result from $O(\delta)$ to $O(\delta^2)$ in the smooth+convex case:
\begin{proposition}
\label{prop:flaxman} Let \cref{ass:gradbasic} hold and let $\gamma$ be the one-point feedback oracle defined in \eqref{eq:one-point}.
Define
$
V = n_W(U)\dfrac{\lvert \partial W\rvert}{\lvert W \rvert}\,,
$
where $W \subset \R^n$ is a convex body with boundary $\partial W$, $U$ is uniformly distributed on $\partial W$, $n_W(U)$ denotes the normal vector of $\partial W$ at $U$, and $\lvert \cdot \rvert$ denotes the appropriate volume. Let $C_2>0$ be defined as in \cref{prop:grad-onepoint}.
Then, if $f$ is $L_0$-Lipschitz, $\gamma$ is a memoryless, uniform type-II oracle with $c_1(\delta)=C_1 \delta$, $c_2(\delta) = C_2/\delta^2$ where
$C_1=L_0 \sup_{w \in W}\|w\|$. 
Further, assuming $W$ is symmetric w.r.t. the origin, if $f$ is $L$-smooth, then $\gamma$ is a type-I (and type-II oracle) with $c_1(\delta) = C_1\delta^2$, $c_2(\delta) = C_2/\delta^2$ where $C_1=(L/|W|)\int_W\|w\|^2 dw$, and, if in addition $f$ is also convex (i.e., $f\in \F_{L,0}$) then $\gamma$ is a type-I oracle with $c_1(\delta)=C_1 \delta^2/2$ and $c_2(\delta)=C_2/\delta^2$.
\end{proposition}
Note that the improvement did not even require convexity.
Also, the bias is smaller for smoother functions, a property that will be enjoyed by all the gradient estimators.

\paragraph{Two-point feedback}
While the one-point estimators are intriguing,
in the optimization setting one can also always group two consecutive observations and obtain similar smoothing-type estimates 
at the price of reducing the number of rounds by a factor of two only, 
which does not change the rate of convergence.
Next we present an oracle that uses two function evaluations to obtain a gradient estimate.
As will be discussed later, this oracle encompasses several simultaneous perturbation methods (see \citealp{bhatnagar-book}):
Given the inputs $x\in \K$,  $0<\delta\le 1$,
the gradient estimate is
\begin{align}
G &=  \dfrac{Z^+ - Z^-}{2\delta}\, V \,, 
 \label{eq:twosp}
\end{align}
where $Z^{\pm} = f(X^{\pm}) + \xi^{\pm}$, $X^{\pm} = x \pm \delta U$, $U,V\in \R^d$, $\xi^{\pm}\in \R$ are random, jointly distributed random variables, $U,V$ chosen by the oracle in the uncontrolled case and chosen by the algorithm in the controlled case
from some fixed distribution characterizing the oracle (depending on $F$), and $\xi^{\pm}$ being the noise of the returned feedback $Z^{\pm}$ at points $X^{\pm}$.
For the following proposition we consider $4=2\times 2$ cases.
First, the function is either assumed to be $L$-smooth and convex (i.e., the derivative of $f$ is $L$-Lipschitz w.r.t. $\norm{\cdot}_*$), or it is assumed to be three times continuously differentiable ($f\in C^3$).
The other two options are either the controlled noise  setting of \eqref{eq:controlled}, or, in the uncontrolled case, we make the alternate assumptions 
%\vspace{-0.2cm}
\begin{align}
\E[\xi^+-\xi^- |\, U,V] &= 0 \text{~~ and ~~}\nonumber\\
\E [ (\xi^{+} - \xi^-)^{2} |\, V] &\le \sigma_\xi^2 <\infty\,.
\label{eq:noiseass}
\end{align}

%%%%%%%%%%%%%%%%%%%%%%%%%%%%%%%%%%%%%%%%%%%%%%%%%%%%%%%%%%%%%%%%%%%%%%%%%%%%%%%%%%%%%%%%%%%%%%%%%%%%%%%%%%%%%%%%%%%%%%%%%%%%%%%%%%%%%%%%%%%%%%%
\begin{table}
\small
\centering
\begin{tabular}{|c|c|c|}
\hline
\textbf{Noise }$\bm{ \rightarrow}$ & \textbf{Controlled } & \textbf{Uncontrolled } \\
\textbf{Function } &(see~\eqref{eq:controlled})&(see~\eqref{eq:noiseass})\\
$\bm{\downarrow}$ &&\\\hline
\textbf{Convex + Smooth}
	& $(C_1 \delta, C_2)$
	& \makecell{ \\[-0.8em] Propositions~\ref{prop:grad-onepoint},\ref{prop:grad-spsa}: $(C_1\delta, \frac{C_2}{\delta^2})$\\[0.2em] \cref{prop:flaxman}: $(C'_1\delta^2, \frac{C_2}{\delta^2})$ \\[0.2em]}
\\
 \hline
 \multirow{2}{*}{$\bm{f \in \C^3}$} 
	& \multirow{2}{*}{$(C_1 \delta^2, \frac{C_2}{\delta^2})$} 
	& \multirow{2}{*}{Propositions~\ref{prop:grad-onepoint},\ref{prop:grad-spsa}: $(C_1 \delta^2, \frac{C_2}{\delta^2})$} \\
 &&\\\hline
\end{tabular}
\caption{Gradient oracles for different function classes and noise categories. %Each table entry specifies the pair $(c_1(\delta), c_2(\delta))$, with 
The constants $C_1, C'_1, C_2$ are defined in Propositions~\ref{prop:grad-onepoint}--\ref{prop:grad-spsa}.
}
\label{tab:oracles}
\end{table}
%%%%%%%%%%%%%%%%%%%%%%%%%%%%%%%%%%%%%%%%%%%%%%%%%%%%%%%%%%%%%%%%%%%%%%%%%%%%%%%%%%%%%%%%%%%%%%%%%%%%%%%%%%%%%%%%%%%%%%%%%%%%%%%%%%%%%%%%%%%%%%%%
The following proposition, whose proof is based on \citep[Lemma~1]{spall1992multivariate} and \citep[Lemma~1]{duchi2015optimal}, provides conditions under which the bias-variance parameters $(c_1,c_2)$ can be bounded as shown in \cref{tab:oracles}:
\begin{proposition}
\label{prop:grad-spsa}
Let \cref{ass:gradbasic} hold and let $\gamma$  be a two-point feedback oracle defined by \eqref{eq:twosp}.
Suppose furthermore that $\E[V U\tr] = I$.
Then $\gamma$ is a type-I oracle with the pair $(c_1(\delta),c_2(\delta))$ 
given by \cref{tab:oracles}. For uncontrolled noise and for controlled noise with $f\in\C^3$, $C_1$ is as in \cref{prop:grad-onepoint} and $C_2$ is $4 C_2$ from \cref{prop:grad-onepoint}. For the controlled noise case with $f \in \F_{L,0}$,
$C_1 = \frac{\overline{L}_{\Psi}}{2} \E[ \dnorm{V} \norm{U}^2]$  
and
$C_2 =  2 B_1^2  + \frac{ \overline{L}_{\Psi}^2}{2}\E\left[ \dnorm{V}^2 \norm{U}^4 \right]$, with $B_1 = \sup_{x\in \K} \scnorm{\nabla f(x)}_*$.
\end{proposition}

\paragraph{Popular choices for $U$ and $V$:}
\begin{itemize}
 \item If we set $U_i$ to be independent, symmetric $\pm 1$-valued random variables and $V_i = 1/U_i$, then we recover the popular SPSA scheme proposed by \citet{spall1992multivariate}.
It is easy to see that $\EE{  V U\tr } = I$ holds in this case.
 When the norm $\norm{\cdot}$ is the $2$-norm, $C_1 = O(d^2)$ and $C_2 = O(d)$. If we set $\norm{\cdot}$ to be the max-norm, $C_1 = O(\sqrt{d})$ and $C_2 = O(d)$.
 \item If we set $V=U$ with $U$ chosen uniform at random on the surface of a sphere with radius $\sqrt{d}$,
 then we recover the RDSA scheme proposed by  \citet[pp.~58--60]{kushcla}.
 In particular, the $(U_i)$ are identically distributed with $\EE{ U_i U_j } = 0$ if $i\ne j$ and $\EE{ U\tr U } = d$, hence $\EE{U_i^2} = 1$. Thus, if we choose $\norm{\cdot}$ to be the $2$-norm, $C_1 = O( d^2 )$ and $C_2 = O(d)$.
 \item If we set $V=U$ with $U$ the standard $d$-dimensional Gaussian with unit covariance matrix, we recover the smoothed functional (SF) scheme proposed by \citet{katkul}.
Indeed, in this case, by definition, $\EE{VU\tr} = \EE{U U\tr } = I$.
When $\norm{\cdot}$ is the $2$-norm, $C_1 = O(d^2)$
 and $C_2 = O( d)$.
 This scheme can also be interpreted as a smoothing operation that  convolves the gradient of  $f$ with a Gaussian density.
\end{itemize}

\subsection{Achievable Results for Stochastic BCO}
We now consider stochastic BCO with $L$-smooth functions over a convex, closed non-empty domain $\K$. 
Let $\cF$ denote the set of these functions.
\citet{duchi2015optimal} proves that the minimax expected optimization error
for the functions $\cF$ with uncontrolled noise is lower bounded by $\Omega(n^{-1/2})$. 
They also give an algorithm which uses two-point gradient estimates which matches this lower bound for the case of \emph{controlled noise}.
For controlled noise, the constructions in the previous section give that for two-point estimators $c_1(\delta) = C_1 \delta^p$ and $c_2(\delta) = C_2\delta^{-q}$ with $p=1$ and $q=0$. Plugging this into
\cref{thm:ub} we get the rate $O(n^{-1/2})$ (which is unsurprising
given that the algorithms and the upper bound proof techniques are essentially the same as that of \citet{duchi2015optimal}).
However, when the noise is uncontrolled, the best that we get is $p=2$ and $q=2$.
From \cref{thm:lb-convex} we get that with such oracles, no algorithm can get better rate than $\Omega(n^{-1/3})$, while from
\cref{thm:ub} we get that these rates are matched by mirror descent.
We can summarize these findings as follows:
\begin{theorem}\label{thm:aaab}
Consider $\F_{L,0}$, the space of convex, $L$-smooth functions over a convex, closed non-empty domain $\K$.
Then, we have the following:\\
\textit{\textbf{Uncontrolled noise}}:
Take any $(\delta^2,\delta^{-2})$ type-I oracle $\gamma$.
There exists an algorithm that uses $\gamma$
and achieves the rate $O(n^{-1/3})$.
Furthermore, no algorithm using $\gamma$
 can achieve better error than $\Omega(n^{-1/3})$ for every $(\delta^2,\delta^{-2})$ type-I oracle $\gamma$.\\
\textit{\textbf{Controlled noise}}:
Take any $(\delta,1)$ type-I oracle $\gamma$.
There exists an algorithm that uses $\gamma$ an
achieves the rate $O(n^{-1/2})$.
Furthermore, no algorithm using $\gamma$
 can achieve better error than $\Omega(n^{-1/2})$ for every $(\delta,1)$ type-I oracle $\gamma$.
\end{theorem}

For stochastic BCO with uncontrolled noise, \citet{AgFoHsuKaRa13:SIAM} analyze a variant of the well-known ellipsoid method and provide regret bounds for the case of convex, $1$-Lipschitz functions over the unit ball. 
Their regret bound implies a minimax error \eqref{eq:minimaxerrdef} bound of order  $O\left(\sqrt{d^{32}/n}\right)$. 
\citet{liang2014zeroth} provide an algorithm based on random walks (and not using gradient estimates) for the setting of convex, bounded functions whose domain is contained in the unit cube and their algorithm results in a bound of the order $O\left((d^{14}/n)^{1/2}\right)$ for the minimax error.
These bounds decrease faster in $n$ than the bound available in \cref{thm:aaab}, while showing a much worse dependence on the dimension.
However, what is more interesting is that our results also shows that an $O(n^{-1/2})$ upper bound \emph{cannot} be achieved solely based on the oracle properties of the gradient estimates considered. Since the analysis of all gradient algorithms for stochastic BCO does this, it is no wonder that the best known upper bound for convex+smooth functions is $O(n^{-1/3})$ \citep{saha2011improved}. (We will comment on the recent paper of \citealt{DeElKo15} later.)

The above result also shows that the gradient oracle based algorithms are optimal for smooth problems, under a controlled noise setting.
While \citet{duchi2015optimal} suggests that it is the power of two-point gradient estimators that helps to achieve this, we need to add that having controlled noise is also critical. %a critical condition to achieve the optimal rate is that the noise must be controlled. 

Finally, let us make some remarks on the early literature on this problem.
A finite time lower bound for stochastic, smooth BCO is presented by  \citet{Chen88:LB-AoS} for
convex functions on the real line. When applied to our setting in the uncontrolled noise case, his results imply that $\EE{ |\hat{X}_n - x^* |}$, that is, the distance of the estimate to the optimum, is at least $\Omega(n^{-1/3})$. 
Note that this is larger than the error achieved by the algorithms of \citet{liang2014zeroth,BubeckDKP15,BuEl15}, but the apparent contradiction is easily resolved by noticing the difference in their error measure: distance to the optimum vs. error in the function value (in particular, compressing the range of functions makes locating the minimizer harder).
 \citet{PoTsy90}, who also considered distance to optimum, proved that mirror descent with gradient estimation achieves asymptotically optimal rates for functions that enjoy high order smoothness.

\section{Applications to Online BCO}
\label{sec:obco}
In the \emph{online BCO} setting a learner sequentially chooses the points $X_1,\dots,X_n\in \cK$ while observing the losses $f_1(X_1),\dots,f_n(X_n)$. More specifically, in round $t$, having observed $f_1(X_1),\dots,f_{t-1}(X_{t-1})$ of the previous rounds, the learner chooses $X_t\in \cK$, after which it observes $f_t(X_t)$. The learner's goal is to minimize its expected regret $\EE{ \sum_{t=1}^n f_t(X_t) - \inf_{x\in \cK} \sum_{t=1}^n f_t(x) }$. 
This problem is also called online convex optimization with one-point feedback.
A slightly different problem is obtained if we allow the learner to choose multiple points in every round, at which points the function $f_t$ is observed. The loss is suffered at $X_t$. The points where the function is observed (``observation points'' for short) may or may not be tied to $X_t$. One possibility is that $X_t$ is one of the observation points.  
Another possibility is that $X_t$ is the average of the observation points (e.g., \citet{AgDeXi10}). Yet another possibility is that there is no relationship between them. 

The oracle constructions from the previous section also apply to the online BCO setting
where the algorithm is evaluated at $Y_t$, though in this case 
one cannot employ two-point feedback as the functions change between rounds. 
This also rules out the controlled noise case. 
Thus, for the online BCO setting, one should consider type-I (and II) oracles with $c_1(\delta) = C_1 \delta^p$ and $c_2(\delta) = C_2\delta^{-q}$ with $p=q=2$.
For these type of oracles, the results from \cref{thm:lb-convex} give the following result: 
\begin{theorem}\label{thm:aaa}
Let $\cF_{L,0}$ be the space of convex, $L$-smooth functions over a convex body $\K$.
No algorithm that relies on 
 $(\delta^2,\delta^{-2})$ type-I oracles
 can achieve better regret than $\Omega(n^{2/3})$.
\end{theorem}
With a noisy gradient oracle of \cref{prop:flaxman}, \cref{thm:aaa} implies that this regret rate is achievable, essentially recovering, and in some sense proving optimality of the result of \citet{saha2011improved}:
\begin{theorem}
For zeroth order noisy optimization with smooth convex functions, the gradient estimator of \cref{prop:flaxman} together with mirror descent (see \cref{alg}) achieve $\O(n^{2/3})$ regret.
\end{theorem}
This ``optimality result'' shows that with the usual analysis of the current gradient estimation techniques, no gradient method can achieve the optimal regret $O(n^{1/2})$ for online bandit convex optimization, established by \citet{BubeckDKP15,BuEl15}. Note that \cref{thm:aaa} contradicts the recent results of \citet{DeElKo15,YaMo16}, who claimed to achieve $\tilde{O}(n^{5/8})$ and, resp., $\tilde{O}(n^{8/13})$ regret with the same $(\delta^2,\delta^{-2})$ type-II gradient oracle as \citet{saha2011improved}, but their proof only used the $(\delta^2,\delta^{-2})$ tradeoff in the bias and variance properties of the oracle.
A thorough inspection of their proofs reveals that Lemma~11 of \citet{DeElKo15} and Lemma~6 of \citet{YaMo16} is incorrect.%
\footnote{Using the notation of \citet{DeElKo15}, in Lemma~11 they used the equation $\E[\hat{g}_{t-i}^T(x_{t-i} - x_t)] = \E[\nabla \hat{f}_{t-i} (x_{t-i})^T(x_{t-i} - x_t)]$, where $\hat{g}_{t-i}$ is the estimated gradient at time $t-i$, $x_t$ is the prediction of their algorithm at time $t$, and $\hat{f}_t$ is an approximate loss function at time t. Since $x_t$ depends on the randomness at time $t-i$, although $\E[\hat{g}_{t-i}] = \nabla \hat{f}_{t-i} (x_{t-i})$, the previous equality does not hold. A similar mistake also appears in Lemma~6 of \citet{YaMo16} (the error is in the last displayed equation in the proof of the lemma). These mistakes mask a bias term, which, if taken into account properly, leads to an $O(n^{2/3})$ regret, in agreement with our lower bound.}
Our attempts to correct their proofs lead to $O(n^{2/3})$ regret, in agreement with our lower bound presented in \cref{thm:aaa}.

\section{Related Work}
\label{sec:related}
Gradient oracle models have been studied in a number of previous papers 
\citep{dAsp08,Baes09,SchRoBa11,DeGliNe14}.
A full comparison between these oracle models is given by \citet{DeGliNe14}.
For illustration, here we only review the model of this latter paper as a typical example of these previous works.
The model of \citet{DeGliNe14} assumes a first-order approximation to the function
with parameters $(\delta,L)$. In particular, 
given $(x,\delta,L)$ and the convex function $f$, 
the oracle gives a pair $(t,g)\in \R \times \R^d$
such that $t + \ip{g,\cdot-x}$ is a linear lower approximation to $f(\cdot)$ in the sense that 
$0\le f(y) - \left\{ t+ \ip{g,y-x}\right\} \le \frac{L}{2} \norm{y-x}^2 + \delta$.
\citet{DeGliNe14} argue that this notion appears naturally in several optimization problems and study whether the so-called accelerated gradient techniques are still superior to their non-accelerated counterparts (and find a negative answer).
The authors study both lower and upper rates of convergence, similarly to our paper.
A major difference between the previous and our settings is that we allow stochastic noise (and bias), which the algorithms can control, while the oracle in these previous paper must guarantee that the accuracy requirements hold in each time step
with probability one.
This is a much stronger requirement, which may be impossible to satisfy in some problems, such as when 
the only information available about the functions is noise contaminated.

Some works, such as \citet{SchRoBa11} allow arbitrary sequences of errors and show error bounds as a function
of the accumulated errors. 
Our proof technique is actually essentially the same (as can be expected).
However, the noisy case requires special care. For example, Proposition~3 of
\citet{SchRoBa11}  bounds the optimization error for the smooth, convex case by 
$O(1/n^2 ( \norm{x_1-x^*}^2 + A_n^2 ))$ where $A_n = O( \sum_{t=1}^n t \norm{e_t})$, $e_t$ being the error of the approximate gradient. This expression becomes $\Theta(\frac{1}{n^2} \sum_{t=1}^n t^2)  \approx n$
%is upper and lower bounded, up to a constant factor by,
assuming that errors' noise level is a positive constant (in all our result, this holds).
This clearly shows that the noisy case requires (somewhat) special treatment.

Similar, but simpler noisy oracle models were introduced \citep{JN11a,Hon12,DvoGa15}, but these models lack the bias-variance tradeoff central to this paper (i.e., they assume the variance and bias can be controlled independently of each other). The results in these papers are upper bounds on the error of certain gradient methods (also to some very specific problem for \cite{Hon12}), and they correspond to the bounds we obtained with $q=0$.

\section{Proof of the Upper Bounds}
\label{sec:appendix-md}
In this section we prove \cref{thm:ub}. First we derive the bounds for the optimization settings and then for the regret. 

\subsection{Stochastic optimization}

The proof for the stochastic optimization scenario is based on \cref{lem:ub} stated below.
This is essentially Theorem~C.4 of \citet{MahdaviPhd:2014}, and also identical to Theorem~6.3 of \citet{Bu:Convex14}, who cites \citet{Dekel:minibatch12} as the source. For completeness, the proof of the lemma is given in \cref{sec:lemub-proof}.
\begin{lemma}
\label{lem:ub}
Let $({\cS}_t)_{t}$ be a filtration such that $X_t$ is ${{\cS}}_t$-measurable.
Let $\overline G_t = \EE{G_t|{{\cS}}_t}$
and assume that the nonnegative real-valued deterministic sequence $(\beta_t)_{1\le t\le n}$ is such that
$\norm{\overline G_t - \nabla {f}(X_t)}_* \le \beta_t$ holds almost surely.
Further, assume that $\mathcal{R}$ is $\alpha$-strongly convex with respect to $\norm{\cdot}$, $D=\sup_{x,y\in \cK} D_{\mathcal{R}}(x,y) < \infty$,  and let $\eta_t = \frac{\alpha}{a_t+L}$ for some increasing
sequence $(a_t)_{t=1}^{n-1}$ of numbers. Then, the cumulative loss of \cref{alg} for a fixed convex and $L$-smooth  function $f$ can be bounded as
\begin{align*}
\EE{ \sum_{t=1}^n {f}(X_t) - {f}(x) }
\le 	 \EE{{f}(X_1)-{f}(x)}+
  \sqrt{\tfrac{2D}{\alpha}} \sum_{t=1}^{n-1} \beta_t
 +\frac{D(a_{n-1}+L)}{\alpha} +
	  \sum_{t=1}^{n-1}\frac{\sigma_t^2}{2a_t}\,,
\end{align*}
where $\sigma_t^2 = \EE{ \norm{G_t-\overline G_t}_*^2}$ is the ``variance'' of $G_t$.

If ${{f}}$ is also $\mu$-strongly convex with respect to $\mathcal{R}$ with $\mu > 2L/\alpha$, then letting $\eta_t = \dfrac{2}{\mu t}$ and $a_t = \alpha \mu t/2-L > 0$, the cumulative loss of  \cref{alg} can be bounded as
\begin{align*}
 \EE{ \sum_{t=1}^n {f}(X_t) - {f}(x) }
\le 	 \EE{{f}(X_1)-{f}(x)}+
 \sqrt{\tfrac{2D}{\alpha}} \sum_{t=1}^{n-1} \beta_t
 +\sum_{t=1}^{n-1}\frac{\sigma_t^2}{2a_t}\,.
\end{align*}
\end{lemma}

Now we can easily prove the theorem.
First we consider the case of smooth and convex functions. We select $$\eta_t = \alpha/(a_t+L)$$ as in the lemma with 
$a_t=a t^r$ for some $0<r<1$. For type-I oracles, the result immediately follows by substituting $\beta_t = C_1\delta^p$, $\sigma^2_t = C_2 \delta^{-q}$, using that $\sum_{t=1}^{n-1} t^{-r} \le 1 +\int_1^n t^{-r} \le n^{1-r}/(1-r)$:
\begin{align}
\MoveEqLeft
\frac{1}{n} \EE{ \sum_{t=1}^n f( X_t) - \inf_{x \in \cK} \sum_{t=1}^n f(x)} \nonumber \\
&\le \frac{1}{n}\left(\EE{f(X_1)-\inf_{x \in \cK}f(x)}+\frac{DL}{\alpha}  \right) +\sqrt{\dfrac{2D}{\alpha}} C_1\delta^p
+\frac{Da}{\alpha} n^{r-1}+\dfrac{C_2 \delta^{-q}}{2a(1-r)}n^{-r} \,.
\label{eq:ubToBeOptTypeI}
 \end{align}
 Choosing 
 \begin{align*}
 r &= \tfrac{p+q}{2p+q},  \\
 a &= 2^{\frac{q}{2(2p+q)}}\left(\tfrac{2p+q}{2p}\right)^{\frac{p}{2p+q}} D^{-\frac{1}{2}} C_1^{\frac{q}{2p+q}} C_2^{\frac{p}{2p+q}} \\
 \delta &= \alpha^{\frac{1}{2(p+q)}}\left(\tfrac{2p+q}{4p}\right)^{\frac{1}{2p+q}} C_1^{-\frac{2}{2p+q}} C_2^{\frac{1}{2p+q}}n^{-\frac{1}{2p+q}},
 \end{align*}
the last $3$ terms in \eqref{eq:ubToBeOptTypeI} are optimized to
 \[
 K_1 D^{1/2} C_1^{q/(2p+q)} C_2^{p/(2p+q)} n ^{-p/(2p+q)} \,,
 \]
 with
 $K_1 = 2^{\frac{q}{2(2p+q)}} \left( \alpha^{-1}+2\alpha^{-\frac{q}{2(p+q)}} \right) \left( \frac{2p+q}{2p} \right)^{\frac{p}{2p+q}}$. This implies \eqref{eq:MDbound1TypeI}.

For type-II oracles, from the bias condition in \cref{def:oracle2} and using that the oracle is memoryless and uniform, we get
\begin{align*}
 \frac{1}{n} \EE{ \sum_{t=1}^n f(X_t) - \inf_{x \in \cK}\sum_{t=1}^n f(x) }
 \le \frac{1}{n}\EE{ \sum_{t=1}^n \tilde{f}(X_t) - \inf_{x \in \cK}\sum_{t=1}^n \tilde{f}(x) } +2C_1 \delta^p
 \,.
\end{align*}

Given $\overline{G}_t=\EE{G_t} = \nabla \tilde{f}(X_t)$, where $\tilde{f}\in \cF_{L,0}$ is convex and smooth,
the result immediately follows by applying \cref{lem:ub} to $\tilde{f}$.
Substituting
 $\beta_t = 0$ (since we have a type-II oracle), $\sigma^2_t = C_2 \delta^{-q}$, respectively, and using the bias condition again, we obtain
 \begin{align}
\MoveEqLeft
\frac{1}{n} \EE{ \sum_{t=1}^n f( X_t) - \inf_{x \in \cK} \sum_{t=1}^n f(x)} \nonumber \\
&\le \frac{1}{n}\left(\EE{\tilde{f}(X_1)-\inf_{x \in \cK}\tilde{f}(x)}+\frac{DL}{\alpha}  \right) %\nonumber
+\frac{Da}{\alpha} n^{r-1}+\dfrac{C_2 \delta^{-q}}{2a(1-r)}n^{-r}+ 2 C_1\delta^p \\
&\le \frac{1}{n}\left(\EE{f(X_1)-\inf_{x \in \cK}f(x)}+\frac{DL}{\alpha}  \right) %\nonumber
+\frac{Da}{\alpha} n^{r-1}+\dfrac{C_2 \delta^{-q}}{2a(1-r)}n^{-r}+ \left(2+\dfrac{2}{n}\right)C_1\delta^p \,.
\label{eq:ubToBeOpt}
 \end{align}
 Choosing 
 \begin{align*} 
 r &= \tfrac{p+q}{2p+q}, \\  
a&= \left(2+\tfrac{2}{n}\right)^{\frac{q}{2p+q}}\left(\tfrac{2p+q}{2p}\right)^{\frac{p}{2p+q}} \left(\tfrac{D}{\alpha}\right)^{-\frac{p+q}{2p+q}}  C_1^{\frac{q}{2p+q}} C_2^{\frac{p}{2p+q}} \\
 \delta &=  \left(2+\tfrac{2}{n}\right)^{-\frac{2}{2p+q}}\left(\tfrac{2p+q}{2p}\right)^{\frac{1}{2p+q}} \left(\tfrac{D}{\alpha}\right)^{\frac{1}{2p+q}}  C_1^{-\frac{2}{2p+q}} C_2^{\frac{1}{2p+q}} n^{-\frac{1}{2p+q}}, 
 \end{align*}
the last $3$ terms in \eqref{eq:ubToBeOpt} are optimized to
 \[
 K'_1 D^{p/(2p+q)} C_1^{q/(2p+q)} C_2^{p/(2p+q)} n ^{-p/(2p+q)} \,,
 \]
where $K_1'= 3\left(2+\frac{2}{n}\right)^{\frac{q}{2p+q}}\left(\frac{2p+q}{2p}\right)^{\frac{p}{2p+q}}  \alpha^{-\frac{p}{2p+q}}$.
 This implies \eqref{eq:MDbound1}.
 
When $\tilde{f} \in \cF_{L,\mu, \cR}$ is $L$-smooth and $\mu$-strongly convex, for
$\eta_t = 2/(\mu t)$ and \\
$
\delta^{p+q} =  \tfrac{C_2\left( \log n+1+\tfrac{\alpha \mu}{\alpha \mu -2L}\right)}{\sqrt{2D\alpha} \mu C_1 n} \,,
$
we similarly obtain, for type-I oracle, 
 \begin{align*}
 \MoveEqLeft
\frac{1}{n} \EE{ \sum_{t=1}^n f( X_t) - \inf_{x \in \cK} \sum_{t=1}^n f(x)} -\frac{1}{n}\EE{f(X_1)-\inf_{x \in \cK}f(x)}\\
&\le \sqrt{\dfrac{2D}{\alpha}} C_1\delta^p+\dfrac{C_2 \delta^{-q}}{\alpha \mu n} \sum_{t=1}^{n-1}\dfrac{1}{t-\dfrac{2L}{\alpha \mu}}\\
&\le \sqrt{\dfrac{2D}{\alpha}}C_1\delta^p+\dfrac{C_2 }{\alpha \mu}\delta^{-q} \dfrac{\log n+1+\alpha \mu/(\alpha \mu-2L)}{n}\\
&\le 2^{\frac{q}{2(p+q)}}\alpha^{-\frac{2p+q}{2(p+q)}}\mu^{-\frac{p}{p+q}} D^{\frac{q}{2(p+q)}}C_1^{\frac{q}{p+q}}C_2^{\frac{p}{p+q}} \left( \frac{\log n+1+\dfrac{\alpha \mu}{\alpha \mu -2L}}{n} \right)^{\frac{p}{p+q}}\,.
 \end{align*}

For type-II oracle, choosing
$
\delta^{p+q} =  \tfrac{C_2\left( \log n+1+\tfrac{\alpha \mu}{\alpha \mu -2L}\right)}{2\alpha \mu C_1 (n+1)} \,,
$
we get
 \begin{align*}
 \MoveEqLeft
\frac{1}{n} \EE{ \sum_{t=1}^n f( X_t) - \inf_{x \in \cK} \sum_{t=1}^n f(x)} -\frac{1}{n}\EE{f(X_1)-\inf_{x \in \cK}f(x)}\\
&\le (2+\dfrac{2}{n})C_1\delta^p+\dfrac{C_2 \delta^{-q}}{\alpha \mu n} \sum_{t=1}^{n-1}\dfrac{1}{t-\dfrac{2L}{\alpha \mu}}\\
&\le (2+\dfrac{2}{n})C_1\delta^p+\dfrac{C_2 }{\alpha \mu}\delta^{-q} \dfrac{\log n+1+\alpha \mu/(\alpha \mu-2L)}{n}\\
&\le 2^{\frac{q}{p+q}}\alpha^{-\frac{p}{p+q}}\mu^{-\frac{p}{p+q}} K_2C_1^{\frac{q}{p+q}}C_2^{\frac{p}{p+q}} \left( \frac{\log n+1+\dfrac{\alpha \mu}{\alpha \mu -2L}}{n} \right)^{\frac{p}{p+q}}\,,
 \end{align*}
where the bound is optimized in the last step via the choice of $\delta$.

\subsection{Online optimization}
The proof in this section follows closely the derivation of \citet{saha2011improved}. First we consider the case of type-II oracles.

Let $\cS_t$ denote the $\sigma$-algebra of all random events up until and including the selection of $X_t$. Since the oracle is unbiased, that is,  $\EE{Y_t|\cS_t}=X_t$, we have
\begin{equation}
\EE{\tf_t(Y_t)-\tf_t(X_t)|\cS_t} \le \EE{\ip{\nabla \tf_t(X_t),Y_t-X_t}+\tfrac{L}{2}\|X_t-Y_t\|^2|\cS_t} \le L\delta^2/2~.
\label{eq:YX}
\end{equation}
This inequality, the definition of type-II oracles, and the convexity of $\tf_t$ implies, for any $x \in \cK$,
\begin{align}
\EE{\sum_{t=1}^n f_t(Y_t)} - \sum_{t=1}^n f_t(x) 
& \le \EE{\sum_{t=1}^n \tf_t(Y_t) -  \sum_{t=1}^n \tf_t(x)} + 2 n C_1 \delta^{p} \nonumber \\
& \le \EE{\sum_{t=1}^n \tf_t(X_t) -  \sum_{t=1}^n \tf_t(x)} + 2 n C_1 \delta^{p} + \frac{n L \delta^2}{2} \nonumber \\
& \le \EE{\sum_{t=1}^n \ip{\nabla \tf_t(X_t), X_t-x }} + 2 n C_1 \delta^{p} + \frac{n L \delta^2}{2} \label{eq:t2-lin} \\
& = \EE{\sum_{t=1}^n \ip{G_t, X_t-x }} + 2 n C_1 \delta^{p} + \frac{n L \delta^2}{2}~.
\label{eq:t2-grad}
\end{align}
Instead of Lemma~\ref{lem:mdlinregret} used in the optimization proof, we apply the prox-lemma \citep[see, e.g.,][]{Beck2003mirror, NeJuLaSh09}:
\begin{equation}
\label{eq:proxlemma}
\ip{G_t,X_t-x} \le \frac{1}{\eta_t}\big(D_\cR(x,X_t)-D_\cR(x,X_{t+1})\big) + \eta_t \frac{\|G_t\|_*^2}{2\alpha}~.
\end{equation}
Summing up the above bound for all $t$, the divergence terms telescope, since
\begin{align}
\lefteqn{\sum_{t=1}^{n-1} \frac{1}{\eta_t} \left(\DR(x,X_t)-\DR(x,X_{t+1})\right)}
 \nonumber \\
&= \DR(x,X_1) \frac{1}{\eta_1} + \DR(x,X_2) \left(\frac{1}{\eta_2}-\frac{1}{\eta_1}\right)
+ \ldots+\DR(x,X_{n-1}) \left(\frac{1}{\eta_{n-1}} -\frac{1}{\eta_{n-2}}\right)\nonumber\\
& \quad- \frac{1}{\eta_{n-1}} \DR(x,X_n) \nonumber \\
& \le \frac{D}{\eta_1} + D \sum_{t=2}^{n-1} \left(\frac1{\eta_{t}}-\frac1{\eta_{t-1}}\right) \nonumber \\
& = \frac{D}{\eta_{n-1}}\,,  \label{eq:div-telescope}
\end{align}
where the inequality results from the fact that $\{\eta_t\}$ is non-increasing.

To bound the last term in \eqref{eq:proxlemma}, we use the assumption $\|\nabla \tf(x)\|_* \le M$ for all $x \in \cK$ to obtain
\begin{equation}
\label{eq:t2-g2}
\EE{\norm{G_t}_*^2 | \cS_t} \le 2\EE{\norm{G_t - \nabla \tf_t(X_t)}_*^2 + \norm{\nabla \tf_t(X_t)}_*^2 \Big| \cS_t}
\le 2(M^2 + C_2 \delta^{-q}),
\end{equation}
Combining the latter with \eqref{eq:t2-grad}, \eqref{eq:proxlemma}, and \eqref{eq:div-telescope}, we obtain, for any $x \in \cK$,
\begin{align}
\EE{\sum_{t=1}^n f_t(Y_t)} - \sum_{t=1}^n f_t(x) 
& \le \frac{D}{\eta_{n-1}}  + \sum_{t=1}^n \eta_t \frac{M^2 + C_2 \delta^{-q}}{\alpha} +  2 n C_1 \delta^{p} + \frac{n L \delta^2}{2}~.
\label{eq:t2-full}
\end{align}
Setting the parameters 
$\delta=(\frac{q}{2p'})^{\frac{2}{2\hp+q}} (\frac{C_2 D}{\alpha \hC_1^2})^{\frac{1}{2\hp+q}} n^{-\frac{1}{2\hp+q}}$, 
where $\hp=\min\{p,2\}$, $\hC_1=C_1 \indic{p \le 2} + (L/4) \indic{p \ge 2}$ (i.e., $\hp$ is the dominating exponent from $\delta^p$ and $\delta^2$, and $\hC_1$ is the coefficient of the dominating term), $\eta_t=D^{\frac{\hp+q}{2\hp+q}} (\frac{q}{2\hp})^{\frac{q}{2\hp+q}} (\frac{C_2}{\alpha})^{-\frac{\hp}{2\hp+q}} \hC_1^{-\frac{q}{2\hp+q}} n^{-\frac{\hp+q}{2\hp+q}}$ gives, when $f_t \in \F_{L,0}$ for all $t$, 
\begin{equation}
\frac{1}{n}\left(\EE{\sum_{t=1}^n f_t(Y_t)} - \inf_{x \in \cK} \sum_{t=1}^n f_t(x) \right) = O\left( \hC_1^{\frac{q}{2\hp+q}} (C_2 D)^{\frac{\hp}{2\hp+q}} n^{-\frac{\hp}{2\hp+q}} \right)
\label{eq:t2-bound}
\end{equation}
where the coefficient of the main term equals $K=2^{1+\frac{q/2}{2\hp+q}}(2\hp+q) (2\hp \alpha)^{-\frac{\hp}{2\hp+q}} q^{-\frac{q}{2\hp+q}}$.

When the set of functions is also strongly convex, in \eqref{eq:t2-lin} we can use strong convexity instead of linearization:
\[
\tf(X_t) - \tf(x) \le \ip{\nabla \tf_t,X_t-x} - \frac{\mu}{2} D_\cR(x,X_t) = \EE{\ip{ G_t,X_t-x}| \cS_t} - \frac{\mu}{2} D_\cR(x,X_t) ~.
\]
Combining this with \eqref{eq:proxlemma} and \eqref{eq:t2-g2} gives the well-known variant of \eqref{eq:t2-full} for strongly convex loss functions \citep{BaHaRa07} for the choice $\eta_t=2/(t\mu)$:
\begin{align*}
\EE{\sum_{t=1}^n f_t(Y_t)}- \sum_{t=1}^n f_t(x) 
& \le \sum_{t=1}^n  \frac{\EE{\|G_t\|_*^2}}{ t \alpha \mu} + 2 n C_1 \delta^{p} + \frac{n L \delta^2}{2} \\
&\le \frac{\max_t \EE{\|G_t\|_*^2}}{\alpha \mu} (1+\log n) + 2 n C_1 \delta^{p} + \frac{n L \delta^2}{2} \\
&\le \frac{2(M^2 + C_2 \delta^{-q})}{\alpha \mu} (1+\log n)+  2 n C_1 \delta^{p} + \frac{n L \delta^2}{2}~.
\end{align*}
Setting $\delta=(\frac{C_2 q (1+\log n)}{\alpha \mu \hC_1 \hp n})^{\frac{1}{\hp+q}}$, we obtain 
\begin{equation}
\label{eq:t2-sc}
\frac{1}{n}\left(\EE{\sum_{t=1}^n \tf_t(Y_t)} - \inf_{x \in \cK} \sum_{t=1}^n \tf_t(x)\right) 
= O\left(\hC_1^{\frac{q}{\hp+q}} C_2^{\frac{\hp}{\hp+q}} n^{-\frac{\hp}{\hp+q}} (1+\log n)^{\frac{\hp}{\hp+q}} \right),
\end{equation}
where the coefficient of the leading term is $K'=(\hp+q)\hp^{-\frac{\hp}{\hp+q}} q^{-\frac{q}{\hp+q}} (\alpha \mu)^{-\frac{\hp}{\hp+q}}$.

For a type-I oracle, we need a slightly different derivation. Using the oracle's definition, similarly to \eqref{eq:t2-grad}, we get for evry $x \in \cK$,
\begin{align}
\EE{\sum_{t=1}^n f_t(Y_t)} - \sum_{t=1}^n f_t(x) 
& \le \EE{\sum_{t=1}^n f_t(X_t) -  \sum_{t=1}^n f_t(x)} + \frac{n L \delta^2}{2} \nonumber \\
& \le \EE{\sum_{t=1}^n \ip{\nabla f_t(X_t), X_t-x }} + \frac{n L \delta^2}{2} \nonumber \\
& = \EE{\sum_{t=1}^n \ip{G_t, X_t-x } + \ip{\nabla f_t(X_t)-G_t,X_t-x}} + \frac{n L \delta^2}{2}  \nonumber \\
& \le \EE{\sum_{t=1}^n \ip{G_t, X_t-x }} + C_1 \delta^p \sum_{t=1}^n \EE{\|X_1-x\|} +  \frac{n L \delta^2}{2} \nonumber \\
& \le \EE{\sum_{t=1}^n \ip{G_t, X_t-x }}  + 2 n R C_1 \delta^p +  \frac{n L \delta^2}{2},
\label{eq:t1-grad}
\end{align}
where the second to last inequality holds by the Cauchy-Schwarz inequality, and in the last step we used our assumption that $\sup_{x \in \cK} \|x\| \le R$. We now proceed similarly to the type-II case, applying the prox-lemma \eqref{eq:proxlemma}, but bound the second moment of $G_t$ differently:
\begin{equation}
\label{eq:t1-g2}
\EE{\norm{G_t}_*^2 | \cS_t} \le 2\EE{\norm{G_t - \EE{G_t | \cS_t}}_*^2} + 2 \norm{\EE{G_t|\cS_t}-\nabla f_t(X_t)}_*^2 
\le 2(C_1^2 \delta^{2p}+ C_2 \delta^{-q}),
\end{equation}
Combining  this with \eqref{eq:proxlemma}, \eqref{eq:div-telescope}, and \eqref{eq:t1-grad} yields
\begin{align*}
\EE{\sum_{t=1}^n f_t(Y_t)} - \sum_{t=1}^n f_t(x) 
& \le \frac{D}{\eta_{n-1}}+   \frac{C_1^2 \delta^{2p}+ C_2 \delta^{-q}}{\alpha} \sum_{t=1}^n \eta_t + 2 n R C_1 \delta^p +  \frac{n L \delta^2}{2} ~.
\end{align*}
Now, the main terms in the above inequality are identical to those of \eqref{eq:t2-full} except that instead of $C_1$ we have $RC_1$ here. Thus, optimizing the parameters of the algorithm for this case, \eqref{eq:t2-full} holds for non-strongly convex loss functions with $\hC_1=R C_1 \indic{p \le 2} + (L/4) \indic{p \ge 2}$. Similarly, \eqref{eq:t2-sc} holds with the latter choice of $\hC_1$ for $\mu$-strongly convex loss functions and type-I oracles.

\section{Proof of the Lower Bounds}
\label{sec:appendix-lb-proof}
%%%%%%%%%%%%%%%%%%%%%%%%%%%%%%%%%%%%%%%%%%%%%%%%%%%%%%%%%%%%%%%%%%%%%%%%%%%%%%% 
In this section we present the proof of \cref{thm:lb-convex}. 
Note that we will only prove lower bounds with the type-I oracle. According to \cref{thm:typered}, lower bounds for type-II can be directly attained by replacing $C_1$ of type-I with $C_1/\sqrt{d}$, given $[+1,-1]^d\subset \cK $.

\subsection{Proof of \cref{thm:lb-convex} for the class of smooth convex functions $\F_{L,0}(\K)$}
\label{sec:appendix-lbconvex}
We will use a novel technique that will allow us to  reduce the $d$-dimensional case to the one-dimensional case (see later).
Thus, we start with the one-dimensional case.

\begin{proof} 
We first prove the theorem for $\F = \F_{L,0}(\K) \cap\{ f:\R \to\R\,:\, \dom(f)=\cK\}$, where 
by the assumptions of the theorem, $L\ge 1/2$, $\K$ is convex and $[-1,1]\subset \K$, thereby proving a slightly
stronger result than stated.
For brevity, let $\Delta_n^{*}$ denote the minimax error $\Delta_n^*(\F, c_1,c_2)$. 
Throughout the proof, a $d$-dimensional normal distribution with mean $\mu$ and covariance matrix $\Sigma$ is denoted by $\normal(\mu, \Sigma)$.

We follow the standard proof technique of lower bounds: We define two functions $f_+, f_- \in \F$ with associated type-I gradient oracles $\gamma_+,\gamma_-$ such that the expected error of any deterministic algorithm can be bounded from below for the case when the environment is chosen uniformly at random from\\ $\{(f_+,\gamma_+),(f_-,\gamma_-)\}$. By Yao's principle
\citep{Yao77:FOCS}, the same lower bound applies to the minimax error $\Delta_n^{*}$ even when randomized algorithms are also allowed.

The proof uses $(c_1,c_2)$ type-I oracles which have no memory.
In particular, we restrict the class of oracles to those that on input $(x,\delta)$ return 
a random gradient estimate 
\begin{equation}
\label{eq:oracle}
G(x,\delta) = \overline{\gamma}(x,\delta) + \xi
\end{equation}
with some map $\og: \cK \times [0,1)\to \R$,
where $\xi$ is a zero-mean normal random variable with variance $c_2(\delta):= C_2 \delta^{-q}$, satisfying the variance requirement, and drawn independently every time the oracle is queried.%
\footnote{The argument presented below is not hard to extend to the case when all observations are from a bounded set,
but this extension is left to the reader.}
The map $\og$, which will be chosen based on $f$ to satisfy the requirement on the bias.
The $Y$ value returned by the oracles is made equal to $x$.

Next we define the two target functions and their associated oracles. With a slight abuse of notation, we will use interchangeably the subscripts $+$ ($-$) and $+1$ ($-1$) for any quantities corresponding to these two environments, e.g., $f_+$ and $f_{+1}$ (respectively, $f_-$ and $f_{-1}$).
For $v \in \{\pm 1\}$, let
\begin{align}\label{eq:fvdef}
f_v(x) :=  \epsilon\left( x-v\right)+2\epsilon^2 \ln\left(1+e^{-\frac{x-v}{\epsilon}}  \right)\,,
%\text{ and } f_-(x) := \epsilon\left( x+1\right)+2\epsilon^2 \ln\left(1+e^{-\frac{x+1}{\epsilon}}  \right), 
\,\, x \in \cK\,.
\end{align}
These functions, with the choice $\epsilon=0.1$, are shown in \cref{fig:lbfvdiff}. 
The idea underlying these functions is that they approximate $\epsilon|x-v|$, but with a prescribed smoothness.
The first and second derivatives of $f_v$ are
\begin{align*}
f'_v(x) &=\epsilon\, \dfrac{1-e^{-\frac{x-v}{\epsilon}}}{1+e^{-\frac{x-v}{\epsilon}}} \,, \qquad \text{ and } \qquad
f''_v(x) = \dfrac{2e^{-\frac{x-v}{\epsilon}} }{\left(  1+e^{-\frac{x-v}{\epsilon}}\right)^2}  \,
\end{align*}
(the functions were designed by choosing $f'_v$).
%%%%%%%%%%%%%%%%%%%%%%%%%%%%%%%%%%%%%%%%%%%%%%%%%%%%%%%%%%%%%%%%%%%%%%%%%%%%%%% 
From the above calculation, it is easy to see that $0 \le f''(x) \le 1/2$; thus $f_v$ is $\frac{1}{2}$-smooth, and so $f_v\in \F$.

For $f_v, v\in\{-1,+1\}$, the gradient oracle we consider is defined as $\gamma_v(x,\delta)=\og_v(x,\delta)+\xi_\delta$ with $\xi_\delta \sim \normal(0,\frac{C_2}{\delta^q})$ selected independently for every query, where $\og_v$ is a biased estimate of the gradient $f'_v$. 
The derivatives of $f_+$ and $f_-$ are shown in \cref{fig:lbfvdiff}; we define the ''bias'' in $\og_v$ to move the gradients closer to each other:
The idea is to shift $f_+'$ and $f_-'$ towards each other, with the shift depending on the allowed bias $c_1(\delta) = C_1\delta^p$.
In particular, since $f_+'\le f_-'$, $f_+'$ is shifted up, while $f_-'$ is shifted down. 
However, the shifted up version of $f_+'$ is clipped for positive $x$ so that it never goes above the 
shifted down version of $f_-'$, cf. \cref{fig:fprime}.
 By moving the curves towards each other, algorithms which rely on the obtained oracles
will have an increasingly harder time (depending on the size of the shift) to distinguish whether the function optimized is $f_+$ or $f_-$. Since
\begin{align*}
0\le f_-'(x) - f_+'(x) \le \sup_{x} f_-'(x) - \inf_x f_+'(x) = 2\epsilon\,,
\end{align*}
we don't allow shifts larger than $\epsilon$ (so no crossing over happens), leading to the following formal definitions:
\begin{align}
\overline{\gamma}_+(x,\delta) = 
	\begin{cases}
	f_+'(x) + \min(\epsilon,C_1\delta^p)\,, & \text{if } x<0\,; \\
	\min\big\{f_+'(x) + \min(\epsilon,C_1\delta^p), f_-'(x) - \min(\epsilon,C_1\delta^p)\big\}\,, & \text{otherwise}\,,
	\end{cases}
	\label{eq:og1}
\end{align}
and
\begin{align}
\overline{\gamma}_-(x,\delta) = 
	\begin{cases}
	f_-'(x) - \min(\epsilon,C_1\delta^p)\,, & \text{if } x>0\,; \\
	\max\big\{f_-'(x) - \min(\epsilon,C_1\delta^p), f_+'(x) + \min(\epsilon,C_1\delta^p)\big\}\,, & \text{otherwise}\,.
	\end{cases}
	\label{eq:og2}
\end{align}
We claim that the oracle $\gamma_v$ based on these functions
 is indeed a $(c_1,c_2)$ type-I oracle, with $c_1(\delta)=C_1\delta^p$ and $c_2(\delta)=\frac{C_2}{\delta^q}$. The variance condition is trivial.
To see that $c_1(\delta) = C_1\delta^p$ works, 
notice that $\gamma_v(x,\delta) = -\gamma_{-v}(-x,\delta)$ and $f_v'(x) = -f_{-v}'(-x)$. Thus,
$|\overline{\gamma}_+(x,\delta)-f_+'(x)| = |\overline{\gamma}_-(-x,\delta)-f_-'(-x)|$, hence it suffices to consider $v=+1$.
The bias condition trivially holds for $x<0$. For $x\ge 0$, using that $f'_+(x) \le f'_-(x)$, we get
$f'_+(x) - \min(\epsilon,C_1\delta^p) \le \og_+(x,\delta) \le f'_+(x) + \min(\epsilon,C_1\delta^p)$, showing 
$|\overline{\gamma}_+(x,\delta)-f_+'(x)|  \le C_1 \delta^p$.
Thus, $\gamma_v$ is indeed an oracle with the required properties.

 \begin{figure}
   \centering
	\begin{tabular}{cc}
	  \subfigure[Plot of $f_+$ and $f_-$ with $\epsilon=0.1$]{
\label{fig:lbfvdiff}
 \centering
   \scalebox{0.8}{\begin{tikzpicture}
   \begin{axis}[width=9cm,height=5cm,
            axis y line=middle,
            axis x line=bottom,
						xlabel={$x$},
						every axis x label/.style={at={(current axis.right of origin)},anchor=west},
						ymax=0.5,
						xtick={-1,0,1},
            xticklabels={-1,0,+1},
            yticklabels=\empty
            ]
           \addplot[domain=-2.75:5, green!35!black, thick,smooth] 
              {0.1*(x-1) + 2*0.01*ln(1+exp(-10*(x-1)))} node [pos=0.9,pin={135:$f_+$}] {} node [pos=0.1,pin={85:{\makecell{\scriptsize decreasing\\\scriptsize when $x\!<\!0$}}}] {}; 
            \addplot[domain=-5:2.75, red!35!black,thick,smooth] 
              {0.1*(x+1) + 2*0.01*ln(1+exp(-10*(x+1)))} node [pos=0.1,pin={45:$f_-$}] {} node [pos=0.615,pin={5:\makecell{\scriptsize$\min\limits_{x<0} f_+(x)$}}] {}; 
   \end{axis}
   \end{tikzpicture}}
	  }
	&
\subfigure[Plot of $f'_+$ and $f'_-$ with $\epsilon=0.1$. 
	The dashed lines show
	$\overline{\gamma}_v(\cdot,\delta)$ for $C_1\delta^p=\epsilon$, 
	$v\in \{\pm 1\}$.]{
		\label{fig:fprime}
 \centering
   \scalebox{0.8}{\begin{tikzpicture}
   \begin{axis}[width=9cm,height=5cm,
            axis y line=middle,
            axis x line=middle,
						xlabel={$x$},
						every axis x label/.style={at={(current axis.right of origin)},anchor=west},
            yticklabels=\empty,
            samples=200
            ]
           \addplot[domain=-5:5, green!35!black, thick,smooth] 
              {0.1*(1-exp(-10*(x-1)))/(1+exp(-10*(x-1)))} node [pos=0.59,pin={0:$f'_+$}] {}; 
            \addplot[domain=-5:5, green!35!black, thick,dashed, smooth] 
              {
              (x<0)*
               (0.1*(1-exp(-10*(x-1)))/(1+exp(-10*(x-1)))+0.05)
               +
               (x>=0)*min(
	               0.1*(1-exp(-10*(x+1)))/(1+exp(-10*(x+1)))-0.05,
    		           0.1*(1-exp(-10*(x-1)))/(1+exp(-10*(x-1)))+0.05
               ) 
              } ;
            \addplot[domain=-5:5, red!35!black,thick,smooth] 
              {0.1*(1-exp(-10*(x+1)))/(1+exp(-10*(x+1)))} node [pos=0.4,pin={135:$f'_-$}] {}; 
              \addplot[domain=-5:5, red!35!black,thick,dashed] 
              {
              (x>0)*
              (0.1*(1-exp(-10*(x+1)))/(1+exp(-10*(x+1)))-0.05)
              +
              (x<=0)*max(
	               0.1*(1-exp(-10*(x+1)))/(1+exp(-10*(x+1)))-0.05,
    		           0.1*(1-exp(-10*(x-1)))/(1+exp(-10*(x-1)))+0.05
              )
              };
   \end{axis}
   \end{tikzpicture}}
}
	\end{tabular}
\end{figure}

To bound the performance of any algorithm in minimizing $f_v, v \in \{\pm 1\}$, notice that $f_v$ is minimized at $x^*_v = v$, with $f_v(v) = 2 \epsilon^2 \ln 2$.
Next we show that if $x$ has the opposite sign of $v$, the difference $f_v(x)-f_v(x_v^*)$ is ``large''.
This will mean that if the algorithm cannot distinguish between $v=+1$ and $v=-1$, it necessarily chooses a
highly suboptimal point for either of these cases.

Since $v f_v$ is decreasing on $\{x\,:\, xv \le 0\}$, we have
\begin{align*}
M_v :=&\,\, \min_{x:xv \le 0} f_v(x) - f_v(v) 
=   f_v(0) - f_v(v)  
=  \epsilon\left(-v + 2\epsilon\ln\dfrac{1+e^{\frac{v}{\epsilon}}}{2}\right). \nonumber %\label{eq:lbfvdiff2}
\end{align*}
Let $h(v) = -v + 2\epsilon\ln\dfrac{1+e^{\frac{v}{\epsilon}}}{2}$.
Simple algebra shows that $h$ is an even function, that is, $h(v) = h(-v)$. Indeed,
\begin{align*}
h(v) = -v + 2\,\epsilon\,\ln\left(e^{\frac{v}{\epsilon}} \dfrac{1+e^{-\frac{v}{\epsilon}}}{2}\right)
= -v + 2\,\epsilon\, \dfrac{v}{\epsilon}  + 2\,\epsilon\,\ln\dfrac{1+e^{-\frac{v}{\epsilon}}}{2}
=  h(-v)\,.
\end{align*}
Specifically, $h(1) = h(-1)$ and thus
\begin{align*}
M_+= M_- = \epsilon\left(-1 + 2\epsilon\ln\dfrac{1+e^{\frac{1}{\epsilon}}}{2}\right)\,.
\end{align*}
From the foregoing, when $xv \le 0$ and $\epsilon<\dfrac{1}{4\ln 2}$,  we have
\begin{align*}
f_v(x)-f_v(x_v^*) \ge \epsilon\left( -1 +2\epsilon \ln\dfrac{1+e^{\frac{1}{\epsilon}}}{2}  \right)> \dfrac{\epsilon}{2}.
\end{align*}
Hence,
\begin{align}
  f_v(x) - f_v(x^*_v)
  \ge \dfrac{\epsilon}{2}  \indic{x v  < 0}. \label{eq:fv-lb}
\end{align}
Given the above definitions and \eqref{eq:fv-lb}, by Yao's principle, the minimax error \eqref{eq:minimaxerrdef} is lower bounded by
\begin{align}
\MoveEqLeft 
\Delta_n^{*} %\nonumber\\
  \ge  \inf_{\A} \,  \E[f_V(\hat X_n) - \inf_{x \in X}  f_V(x)]
  \ge \inf_{\A} \, \dfrac{\epsilon}{2}\,  \P(\hat X_n V < 0)\,,
  \label{eq:avg-bd}
  \end{align}
where $V \in \{\pm 1\}$ is a random variable, $\hat{X}_n$ is the estimate of the algorithm after $n$ queries to the oracle $\gamma_V$ for $f_V$, the infimum is taken over all deterministic algorithms, and the expectation is taken with respect to the randomness in $V$ and the oracle. More precisely, the distribution above is defined as follows:

Consider a fixed $(c_1,c_2)$ type-I oracle $\gamma$ satisfying \eqref{eq:oracle} and a deterministic algorithm $\A$. Let $x_t^{\A}$ (respectively, $\delta_t^{\A}$) denote the map from the algorithm's past observations that picks the point (respectively, accuracy parameter $\delta$), which are sent to the oracle in round $t$. Define the probability space $(\Omega, \B, P_{\A,\gamma})$ with
$\Omega = \R^n\times \{-1,1\}$,  its associated Borel sigma algebra $\B$, where the probability measure
$P_{\A,\gamma}$ takes the form $P_{\A,\gamma} := p_{\A,\gamma} d(\lambda \times m)$, 
where
	$\lambda$ is the Lebesgue measure on $\R^n$, 
	$m$ is the counting measure on $\{\pm 1\}$ and 
	$p_{\A,\gamma}$ is the density function defined by
\begin{align*}
&p_{\A,\gamma}(g_{1:n}, v) 
= \frac{1}{2} \bigg( p_{\A,\gamma}(g_n \mid g_{1:n-1})
		\cdot \ldots \cdot p_{\A,\gamma }(g_{n-1} \mid g_{1:n-2}) \cdot \ldots \cdot p_{\A,\gamma}(g_1) \bigg) \\
&\!=\!  \frac{1}{2} \bigg( p_{\N}\big(g_n - 
				\overline{\gamma}(x_n^{\A}(g_{1:n-1}),\delta_n^{\A}(g_{1:n-1})),c_2(\delta_n^{\A}(g_{1:n-1}))\big) \cdot
									 \ldots \cdot  p_{\N}\big(g_1 - \overline{\gamma}(x_1^{\A},\delta_1^{\A}),c_2(\delta_1^{\A})\big) \bigg),
\end{align*}
where $v\in\{-1,1\}$ and $p_{\N}(\cdot,\sigma^2)$ is the density function of a $\normal(0,\sigma^2)$ random variable.
Then the expectation in \eqref{eq:avg-bd} is defined w.r.t. the distribution $\P:= \dfrac{1}{2} \left(P_{\A, \gamma_+} \indic{v=+1} + P_{\A, \gamma_-}\indic{v=-1}\right)$ and $V: \Omega \to \{\pm 1 \}$ is defined by $V(g_{1:n},v) = v$.%
\footnote{Here, we are slightly abusing the notation as $\P$ depends on $\A$, but the dependence is suppressed.
In what follows, we will define several other distributions derived from $\P$, which will all depend on $\A$, but
for brevity this dependence will also be suppressed.
The point where the dependence on $\A$ is eliminated will be called to the reader's attention.}
Define $\P_{+}(\cdot) := \P(\cdot\mid V=1)$, $\P_{-}(\cdot) := \P(\cdot\mid
V=-1)$. 
From \eqref{eq:avg-bd}, we obtain
\begin{align}
\Delta_n^{*}  
%\ge & \inf_{\A} \dfrac{\epsilon}{4}\,  \P(\hat X_n V < 0), \label{eq:strong-convex-bd}\\
  \ge & \inf_{\A} \dfrac{\epsilon }{4} \, \left(\P_{+}(\hat X_n < 0) + \P_{-}(\hat X_n > 0)\right), \label{eq:Pplus}\\
  \ge &\inf_{\A} \dfrac{\epsilon }{4} \,\left(1 - \tvnorm{\P_{+}- \P_{-}}\right), \label{eq:lecam}\\
  \ge &\inf_{\A} \dfrac{\epsilon }{4}  \,\left( 1 - \left(\frac12\dkl{P_{+}}{P_{-}}\right)^{\frac{1}{2}}\right), \label{eq:pinsker}
\end{align}
where 
 \eqref{eq:Pplus} uses the definitions of $\P_+$ and $\P_-$, $\tvnorm{\cdot}$ denotes the total variation distance,  
 \eqref{eq:lecam} follows from its definition, while \eqref{eq:pinsker} follows from Pinsker's inequality. 
It remains to upper bound $\dkl{P_{+}}{P_{-}}$.

%%%%%%%%%%%%%%%%%%%%%%%%%%%%%%%%%%%%%%%%%%%%%%%%%%%%%%%%%%%%%%%%%%%%%%%%%%%%%%% 
Define $G_t$ to be the $t$th observation of $\A$. Thus, $G_t:\Omega \to \R$, with $G_t( g_{1:n}, v) = g_t$.
Let $P_+^t(g_1,\dots,g_t)$ denote the joint distribution of $G_1,\dots,G_t$ conditioned on $V=+1$.
Let $P_{+}^t(\cdot\mid g_1,\ldots,g_{t-1})$ denote the distribution of $G_t$ conditional on $V=+1$ and $G_1=g_1,\ldots,G_{t-1}=g_{t-1}$. Define  $P_{-j}^t(\cdot\mid g_1,\ldots,g_{t-1})$ in a similar fashion.
Then, by the chain rule for KL-divergences, we have
\begin{align}
\label{eq:dklchain}
&\dkl{P_{+}}{P_{-}}= \sum_{t=1}^n \int_{\R^{t-1}} \dkl{P_{+}^t(\cdot\mid g_{1:t-1})}{P_{-}^t(\cdot\mid g_{1:t-1})} d P_{+}^t( g_{1:t-1}).
\end{align}
By the oracle's definition on $V=+1$ we have
$G_t \sim  \normal(\overline{\gamma}_{+}(x^{\cA}_t(G_{1:t-1}),\delta_t^{\A}(G_{1:t-1})),c_2(\delta^{\A}_t(G_{1:t-1})))$, i.e., 
$P_{+}^t(\cdot\mid g_{1:t-1})$ is the normal distribution with mean 
$\overline{\gamma}_{+}(x^{\cA}_t(G_{1:t-1}),\delta^{\A}_t(G_{1:t-1}))$ and variance $c_2(\delta^{\A}_t(G_{1:t-1}))$.
Using the shorthands $x_t^{\A}:=x^{\A}_t(g_{1:t-1})$, $\delta_t^{\A}:=\delta^{\A}_t(g_{1:t-1})$,
we have
\begin{align*}
\dkl{P_{+}^t(\cdot\mid g_{1:t-1})}{P_{-}^t(\cdot\mid g_{1:t-1})}
& =\dfrac{(\overline{\gamma}_{+}(x_t^{\A},\delta_t^{\A}) - \overline{\gamma}_{-}(x_t^{\A},\delta_t^{\A}))^2}{2 c_2(\delta^{\A}_t)}\,,
\end{align*}
as the KL-divergence between normal distributions $\normal(\mu_1,\sigma^2)$ and $\normal(\mu_2,\sigma^2)$ is equal to $\dfrac{(\mu_1 - \mu_2)^2}{2 \sigma^2}$.

It remains to upper bound the numerator. For $(x,\delta)\in \R\times (0,1]$, first note that \\$\gamma_+(x,\delta)\le \gamma_-(x,\delta)$. Hence,
\begin{align*}
|\gamma_+(x,\delta)-\gamma_-(x,\delta)|
& =  \gamma_-(x,\delta) - \gamma_+(x,\delta) \\
& < \sup_x \gamma_-(x,\delta) - \inf_x \gamma_+(x,\delta) \\
& = \lim_{x\to\infty} \gamma_-(x,\delta) - \lim_{x\to-\infty} \gamma_+(x,\delta) \\
& = \epsilon - \epsilon\wedge C_1\delta^p - (-\epsilon + \epsilon \wedge C_1\delta^p)\\
& = 2\epsilon - 2\epsilon \wedge C_1\delta^p\\
& \le 2(\epsilon - C_1\delta^p)^+\,, \numberthis \label{eq:gdiff-ub}
\end{align*}
where $(u)^+ = \max(u,0)$ is the positive part of $u$.

From the above, using the abbreviations $x_t^{\A} = x^{\A}_t(g_{1:t-1})$ and $\delta_t^{\A} = \delta^{\A}_t(g_{1:t-1})$ (effectively fixing $g_{1:t-1}$ for this step),
\begin{align}
\dkl{P_{+}^t(\cdot\mid g_{1:t-1})}{P_{-}^t(\cdot\mid g_{1:t-1})}
& < \dfrac{2\{(\epsilon-C_1(\delta^{\A}_t)^p)^+\}^2\,(\delta^{\A}_t)^q}{C_2}\label{eq:dkgauss}\\
& \le  \sup_{\delta>0} \dfrac{2\{(\epsilon-C_1\delta^p)^+\}^2\,\delta^q}{C_2} \label{eq:supdelta}\,,
\end{align}
where inequality \eqref{eq:dkgauss} follows from \eqref{eq:gdiff-ub}. Notice that the right-hand side of the above inequality does not depend on the algorithm anymore.

Now, observe that 
$\sup_{\delta> 0} \{(\epsilon - C_1 \delta^p)^+\}^2 \delta^q = \sup_{(\epsilon/C_1)^{1/p} \ge \delta> 0} (\epsilon - C_1 \delta^p)^2 \delta^q$. From this we obtain
\begin{align}
\delta_*=\left(\frac{\epsilon q}{C_1(2p+q)}\right)^{1/p}. \label{eq:deltastar}
\end{align}
Note that $C_1\delta_*^p\le \epsilon$, hence
$\max_{\delta> 0} \{(\epsilon - C_1 \delta^p)^+\}^2 \delta^q =  (\epsilon-C_1 \delta_*^p)^2 \delta_*^q$.
Plugging \eqref{eq:supdelta} into \eqref{eq:dklchain} and using this last observation we obtain
\begin{align}
\dkl{P_{+}}{P_{-}} \le \dfrac{2n}{C_2} \,(\epsilon-C_1\delta_*^p)^2\, \delta_*^q\,.
\end{align}
Note that the above bound holds uniformly over all algorithms $\A$. 
Substituting the above bound into \eqref{eq:pinsker}, we obtain 
\begin{align}
\Delta_n^{*}
  \ge  \dfrac{\epsilon}{4} \left(1 - \sqrt{
    n}  \dfrac{ (\epsilon-C_1\delta_*^p)\delta_*^{q/2}}{\sqrt{C_2}}
  \right)
  = \frac{\epsilon}{4}\left(1-\sqrt{n} K_1 \epsilon^{\frac{2p+q}{2p}}\right)\,,\label{eq:final-lower-bd}
\end{align}
where $K_1= \frac{2p}{\sqrt{C_2}(2p+q)}\left(\frac{q}{C_1(2p+q)}\right)^{\frac{q}{2p}}$.
 
By choosing $\epsilon = \left(\frac{2p}{\sqrt{n} K_1(4p+q)} \right)^{\frac{2p}{2p+q}}$, we see that 
\begin{align}
\Delta_n^* \ge \frac{2p+q}{4(4p+q)}\left(\frac{2p}{\sqrt{n} K_1(4p+q)} \right)^{\frac{2p}{2p+q}} = \frac{\left(2p+q\right)^2}{4q^{\frac{q}{2p+q}}\left(4p+q\right)^{\frac{4p+q}{2p+q}}}C_1^{\frac{q}{2p+q}} C_2^{\frac{p}{2p+q}}  n^{-\tfrac{p}{2p+q}} \,.
\label{eq:lb-pq}
\end{align}

Now, when $p=1$ and $q=2$, the lower bound in \eqref{eq:lb-pq} simplifies to 
\[
\Delta_n^{*} \ge \dfrac{ 1}{3\sqrt{3}} C_1^{1/2}C_2^{1/4} n^{-1/4} \,.
\]
On the other hand, for $p=q=2$, we obtain 
\[
\Delta_n^{*} \ge  \frac{9}{20}\left(\frac{1}{25}\right)^{1/3}C_1^{1/3}C_2^{1/3} n^{-1/3}\,.
\]
\end{proof}
%%%%%%%%%%%%%%%%%%%%%%%%%%%%%%%%%%%%%%%%%%%%%%%%%%%%%%%%%%%%%%%%%%%%%%%%%%%%%%%
%%%%%%%%%%%%%%%%%%%%%%%%%%%%%%%%%%%%%%%%%%%%%%%%%%%%%%%%%%%%%%%%%%%%%%%%%%%%%%%
%%%%%%%%%%%%%%%%%%%%%%%%%%%%%%%%%%%%%%%%%%%%%%%%%%%%%%%%%%%%%%%%%%%%%%%%%%%%%%%
\paragraph{Generalization to $d$ dimensions:}
To prove the $d$-dimensional result, we introduce a new device which allows us to relate the minimax error of the $d$-dimensional problem to that of the $1$-dimensional problem.
The main idea is to use separable $d$-dimensional functions and oracles and show that if there exists an algorithm with a small loss for a rich set of separable functions and oracles, then there exists good one-dimensional algorithms for the one-dimensional components of the functions and oracles.

This device works as follows: First we define one-dimensional functions.
For $1\le i \le d$, let $\cK_i \subset \R$ be nonempty sets, 
and for each $v_i \in V := \{\pm 1\}$, let $f_v^{(i)}: \cK_i \to \R$.
Let $\cK = \times_{i=1}^d \cK_i$ and for $v = (v_1,\dots,v_d) \in V^d$, let $f_v: \cK \to \R$ be defined by
\begin{align}
f_v(x) = \sum_{i=1}^d f^{(i)}_{v_i}(x_i), \qquad x\in \cK\,. \label{eq:sepfdef}
\end{align}
Without the loss of generality, we assume that $\inf_{x_i \in \cK_i} f_{v_i}^{(i)}(x_i) = 0$, and hence $\inf_{x\in \times_{i=1}^d \cK_i} f_{v}(x) = 0$, so that the optimization error of the algorithm producing $\hat{X}_n \in\cK$  as the output is $f_v^{(i)}(\hat{X}_{n,i})$ and $f_v(\hat{X}_{n})$, respectively.
We also define a $d$-dimensional \emph{separable} oracle $\gamma_v$  as follows: 
The oracle  is obtained from ``composing'' the $d$ one-dimensional oracles, $(\gamma_{v_i}^{(i)})_{i}$.
In particular, the $i$th component of the response of $\gamma_v$ 
given the history of queries $(x_{t},\delta_{t},\dots,x_1,\delta_1)\in (\cK \times [0,1))^t$
is defined as the response of $\gamma^{(i)}_{v_i}$ 
given the history of queries $(x_{t,i},\delta_{t},\dots,x_{1,i},\delta_1)\in (\cK_i\times [0,1))^t$.
This definition is so far unclear about the randomization of the oracles. 
In fact, it turns out that the one-dimensional oracles can even use the same randomization (i.e.,
their output can depend on the same single uniformly distributed random variable $U$), but they could also use separate randomization: our argument will not depend on this.
\newcommand{\sep}{\mathrm{sep}}
Let $\Gamma^{(i)}(f_{v_i}^{(i)},c_1,c_2)$ 
denote a non-empty set of $(c_1,c_2)$ type-I oracles for objective function $f^{(i)}_{v_i}:\cK_i \to \R$,
and let us denote by $\Gamma_\sep(f_v,c_1,c_2)$ the set of separable oracles for the function $f_v$ 
defined above.
We also define $\cF_\sep = \{ f\,: \, f(x) = \sum_{i=1}^d f^{(i)}_{v_i}(x_i), x\in \cK, v_i\in V_i \}$, the set of componentwise separable functions.
Note that when $\norm{\cdot} = \norm{\cdot}_2$ is used in the definition of type-I oracles then
$\Gamma_\sep(f_v,c_1/\sqrt{d},c_2/d) \subset \Gamma(f_v,c_1,c_2)$.

Let an algorithm $\A$ interact with an oracle $\gamma$.
We will denote the distribution of the output $\hat{X}_n$ of $\A$
at the end of $n$ rounds by $F_{\A,\gamma}$ 
(we fix $n$, hence the dependence of $F$ on $n$ is omitted).
Thus, the expected optimization error of $\A$ on a function $f$ with zero optimal value is 
\begin{align*}
L^{\A}(f,\gamma) = \int f(x) F_{\A,\gamma}(d x)\,.
\end{align*}
Note that this definition applies both in the one and the $d$-dimensional cases.
For $v\in V^d$, we introduce the abbreviation
\begin{align*}
L^{\A}(v) = L^{\A}(f_v,\gamma_v)\,.
\end{align*}
We also define
\newcommand{\tL}{\tilde{L}}
\begin{align*}
\tL^{\A}_i(v) = \int f_{v_i}^{(i)}( x_i ) F_{\A,\gamma_v}(d x)\,
\end{align*}
so that 
\begin{align*}
L^{\A}(v) = \sum_{i=1}^d \tL^{\A}_i(v)\,.
\end{align*}
Also, for $v_i \in V$ and a one-dimensional algorithm $\A$, we let
\begin{align*}
L^{\A}_i(v_i) = L^{\A}(f_{v_i}^{(i)}, \gamma_{v_i}^{(i)})\,.
\end{align*}
Note that while the domain of $\tL^{A}_i$ is $V^d$, the domain of $L^{\A}_i$ is $V$,
while both express an expected error measured against $f_{v_i}^{(i)}$.
In fact,  $\tL^{A}_i$ depends on $v$ because the algorithm $\A$ uses the $d$-dimensional oracle $\gamma_v$, which depends on $v$ (and not only on $v_i$) and thus algorithm $\A$ could use information returned by $\gamma_{v_j}^{(j)}$, $j\ne i$. In a way our proof shows that using this information cannot help a $d$-dimensional algorithm on a separable problem, a claim that we find rather intuitive, and which we now formally state and prove.

\begin{lemma}[``Cross-talk'' does not help in separable problems]
\label{lemma:sep}
Let %$V= \times_{i=1}^d V_i$, 
$(f_v)_{v\in V^d}, f_v \in \cF_\sep$, \\$(\gamma_v)_{v\in V^d}, \gamma_v \in \Gamma_\sep(f_v,c_1,c_2)$ be separable for some arbitrary functions $c_1,c_2$, and let $\A$ be any $d$-dimensional algorithm. Then there exist  $d$ one-dimensional algorithms, $\A_i^*$, $1\le i \le d$ (using only one-dimensional oracles),
such that 
\begin{align}
\label{eq:onedimlb}
% \begin{split}
% \MoveEqLeft 
&\max_{v\in V} L^{\A}(v) 
\ge   \max_{v_1\in V_1} L_1^{\A_1^*}(v_1) + \dots + \max_{v_d\in V_d} L_d^{\A^*_d}(v_d)\,.
% \end{split}
\end{align}
\end{lemma}
\begin{figure}
\begin{center}
	\includegraphics[width=0.5\textwidth]{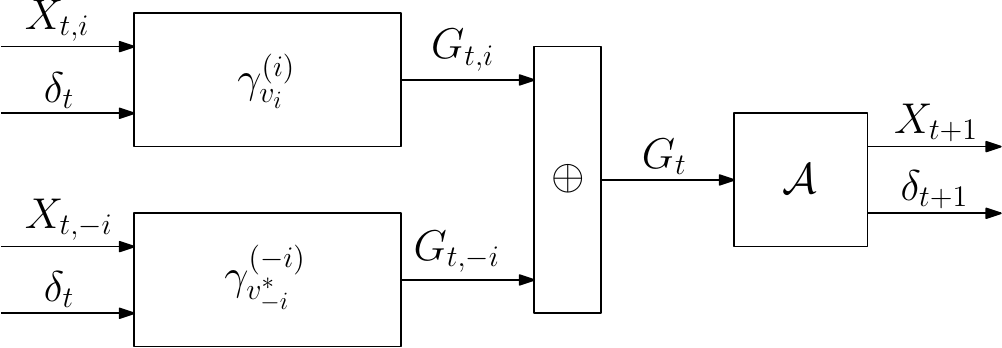} % drawn with the IPE app
\end{center}
\caption{The construction of algorithm $\A_i^*$ used in the proof of \cref{lemma:sep}.}
\label{fig:sepalgconstruction}
\end{figure}
\begin{proof}
We will explicitly construct the one-dimensional algorithms, using $\A$. The difficulty is that $\A$ is $d$-dimensional, and the $i$th one-dimensional algorithms can only interact with the one-dimensional oracle
that depends on $v_i$ but does not depend on  $v_{-i}:=(v_1,\ldots,v_{i-1},v_{i+1},\ldots,v_d)$.
Hence, to use $\A$ we need to supply some values $v_{-i}^*$ replacing $v_{-i}$ so that we can use 
the full $d$-dimensional oracle, which $\A$ needs.

Before the construction, we need one more notational convention: 
Slightly abusing notation, we let $v = (v_i,v_{-i})$ and when writing $(v_i,v_{-i})$ as the argument of some function $g$, instead of $g( (v_i,v_{-i}) )$ we will write $g( v_i,v_{-i})$. The decomposition of a vector into one component and all the others will also be used for other $d$-dimensional vectors (not only for $v\in V$).

To define $\A_i^*$, consider the solution of the following max-min problem:
\[
\max_{v_i} \min_{v_{-i}} \tL^{\A}_i(v_i, v_{-i}))\,.
\]
Let the optimal solution of this problem be denoted by $(\hat{v}_i^*,v_{-i}^*)$; we will use $v_{-i}^*$ replacing the missing values $v_{-i}$ when
we create a one-dimensional oracle from a $d$-dimensional.
We also collect $(\hat{v}_i^*)_i$ into the vector $\hat{v}^*\in V^d$.

Now, algorithm $\A_i^*$ is constructed as illustrated on \cref{fig:sepalgconstruction}.
Fix $v_i\in V_i$. Then, algorithm $\A_i^*$ interacts with oracle $\gamma_{v_i}^{(v_i)}$ as follows:
In each round $t$, algorithm $\A_i^*$ produces a pair $(X_t,\delta_t) \in \cK \times [0,1)$. 
In particular, in the first round, $X_1,\delta_1$ is the output of $\A$ in the first round.
In round $t+1$, given the pair $X_t,\delta_t$ produced in the previous round,
the $i$th component of $X_t$ and $\delta_t$ are fed to oracle $\gamma_{v_i}^{(i)}$ (the $i$th component of oracle $\gamma_v$), 
whose output we name $G_{t,i}$. 
The other components of $X_t$, namely $X_{t,-i}$, together with $\delta_t$ are fed to
oracle $\gamma_{v^*_{-i}}^{(-i)}$ which produces a $d-1$-dimensional vector of all but the $i$th component of $\gamma_{(v_i,v^*_{-i})}$, which we call $G_{t,-i}$. 
The values $G_{t,i}$, $G_{t,-i}$ are put together to form the $d$-dimensional vector 
$G_t = (G_{t,i},G_{t,-i})$, which is fed to algorithm $\A$.
We then set $(X_{t+1},\delta_{t+1})$ to be equal to the output of $\A$. 
At the end of the $n$ rounds, $\A$ is queried to produce $\hat{X}_n$, 
whose $i$th component, $\hat{X}_{n,i}$, 
is returned as the output of $\A_i^*$.

By construction, $L^{\A_i^*}_i(v_i) = \tL_i^{\A}(v_i,v_{-i}^*)$.
Now, notice that 
\begin{align*}
\max_{v_i\in V_i} \tL^{\A}_i(v_i,v^*_{-i}) 
=  \tL^{\A}_i(\hat{v}_i^*,v^*_{-i}) 
\le  \tL^{\A}_i(\hat{v}_i^*,\hat{v}^*_{-i})  = \tL^{A}_i(\hat{v}^*)\,,
\end{align*}
where the equality uses the definition of $\hat{v}_i^*$,
while the inequality uses the definition of $v^*_{-i}$.
Thus,
\begin{align*}
\sum_{i=1}^d \max_{v_i\in V_i} L^{\A_i^*}_i(v_i)
 \le \sum_{i=1}^d  \tL^{\A}_i(\hat{v}^*) 
 = L^{\A}(\hat{v}^*) \le \max_{v\in V } L^{\A}(v)\,,
\end{align*}
which was the claim to be proven.
\end{proof}

Now, let 
\[
\cF^{(i)} = \{f_{v_i} \,:\, v_i\in V\}, \qquad i=1,\dots,d\,.
\] 
The next result follows easily from the previous lemma:
\begin{lemma}
\label{lem:sep2}
Let $\norm{\cdot} =\norm{\cdot}_2$ in the definition of the type-I oracles.
Then, we have that 
\[
\Delta^*_{\cF_\sep,n}(c_1, c_2 ) \ge \sum_{i=1}^d \Delta_{\cF^{(i)},n}^*(c_1/\sqrt{d},c_2/d)\,.
\]
\end{lemma}
\begin{proof}
By our earlier remark, $\Gamma_\sep(f_v,c_1/\sqrt{d},c_2/d) \subset \Gamma(f,c_1,c_2)$. Hence,
\begin{align*}
\Delta^*_{\cF_\sep,n}(c_1, c_2 ) 
	& = \inf_{\cA} \sup_{v\in V} \sup_{\gamma \in \Gamma(f_v,c_1,c_2)} \Delta^{\cA}_n(f_v,\gamma) 
	 \ge \inf_{\cA} \sup_{v\in V} \sup_{\gamma \in \Gamma_\sep(f_v,c_1/\sqrt{d},c_2/d)} \Delta^{\cA}_n(f_v,\gamma)\,.
	 \numberthis
	 \label{eq:ddimtoonedim1}
\end{align*}
For each $i=1,\dots,d$, pick $\gamma_{v_i}^{(i)}\in \Gamma(f_{v_i},c_1/\sqrt{d},c_2/d)$ such that 
$$\Delta_n^*(\cF^{(i)},c_1/\sqrt{d},c_2/d) = \inf_{\cA} \sup_{v_i\in V_i} \Delta_n^{\cA}(f_{v_i},\gamma_{v_i}^{(i)}).$$
For $v \in V$, let $\gamma_v \in \Gamma_\sep(f_v,c_1/\sqrt{d},c_2/d)$ be the oracle whose ``components'' are 
$\gamma_{v_i}^{(i)}$, $i=1,\dots,d$.
Now, by \cref{lemma:sep},
\begin{align*}
\sup_{v\in V} \Delta^{\cA}_n(f_v,\gamma_v)
\ge
\sum_{i=1}^d \inf_{\cA} \sup_{v_i\in V_i} \Delta^{\cA}_n(f_{v_i}^{(i)},\gamma_{v_i}^{(i)})
=
\sum_{i=1}^d  \Delta_{\cF^{(i)},n}^*(c_1/\sqrt{d},c_2/d)\,.
\end{align*}
This, together with 
$\sup_{v\in V} \sup_{\gamma \in \Gamma_\sep(f_v,c_1/\sqrt{d},c_2/d)} \Delta^{\cA}_n(f_v,\gamma)
\ge \sup_{v\in V} \Delta^{\cA}_n(f_v,\gamma_v)$
and \eqref{eq:ddimtoonedim1} gives the desired result.
\end{proof}

\begin{proof}(\cref{thm:lb-convex})
Let $\cK \subset \R^d$, such that $\times_i \cK_i \subset \cK$, $\{\pm 1 \} \subset \cK_i \subset \R$,
$\F_d = \F_{L,0}(\K)$, where recall that $L\ge 1/2$.
 For any $1\le i \le d$, $x_i \in \cK_i$, 
\begin{align}
\label{eq:smoothddim}
  f^{(i)}_{v_i}(x_i) := \epsilon\left( x_i-v_i\right)+2\epsilon^2 \ln\left(1+e^{-\frac{x_i-v_i}{\epsilon}}  \right)\,.
\end{align}
i.e., $f^{(i)}_{v_i}$ is like in the one-dimensional lower bound proof (cf. equation~\ref{eq:fvdef}).
Note that $f_v \in \F_d$ since $f_v$ is separable, so its Hessian is diagonal and from our earlier calculation
we know that
$0\le \frac{\partial^2}{\partial x_i^2} f^{(i)}_{v_i}(x_i) \le 1/2$.
Let $\Delta_n^{(d)*}$ denote the minimax error $\Delta_{\F_d, n}^*\left(C_1\delta^p,\frac{C_2}{\delta^q}\right)$ for the $d$-dimensional family of functions $\F_d$. 
Let $\F^{(i)} = \{ f^{(i)}_{-1},  f^{(i)}_{+1} \}$.
As it was noted above, $f_v\in \F_d$ for any $v\in \{\pm 1\}^d$.
Hence, by~\cref{lem:sep2}, 
\begin{align*}
 \Delta_n^{(d)*} &\ge \sum_{i=1}^d \Delta_{\F^{(i)},n}^{*}\left(\frac{C_1}{\sqrt{d}}\, \delta^p, \frac{C_2}{d} \delta^{-q}\right)\,.
               \numberthis \label{eq:dlb}
\end{align*}

\paragraph{Derivation of rates:}\ \\
Plugging the lower bound derived in \eqref{eq:lb-pq} for the one-dimensional setting into the bound in \eqref{eq:dlb}, we obtain a $\sqrt{d}$-times bigger lower bound for the $d$-dimensional case for any $p, q >0$:
\begin{align}
\Delta_n^{(d)*} \ge  \sqrt{d} \frac{\left(2p+q\right)^2}{2q^{\frac{q}{2p+q}}\left(4p+q\right)^{\frac{4p+q}{2p+q}}}C_1^{\frac{q}{2p+q}} C_2^{\frac{p}{2p+q}}  n^{-\tfrac{p}{2p+q}} \,.
\label{eq:lb-pq-1}
\end{align}
The above bound simplifies to the following for the case where $p=1$ and $q=2$:
\begin{align*}
\Delta_n^{(d)*}  \ge& \dfrac{ 2(C_1^2C_2)^{1/4}}{3\sqrt{3}} \sqrt{d}n^{-1/4}.
\end{align*}

On the other hand, for the case $p=q=2$, we obtain
\begin{align*}
\Delta_n^{(d)*}  \ge& \frac{9}{10}\left(\frac{C_1 C_2}{25}\right)^{1/3} \sqrt{d}n^{-1/3}.
\end{align*}
\end{proof}

\subsection{Proof of \cref{thm:lb-convex} for strongly convex and smooth function class $\F_{L,1}(\K)$.}
\label{sec:appendix-lbscconvex}
\begin{proof}
We follow the notational convention used earlier for convex functions in one dimension. 
Let $\F = \F_{L,1}(\K)$, where $L\ge 1$ and $\K$  contains $\pm 1$.
We consider functions $f_v$, for $v \in \{-1,+1\}$, defined as
\begin{align}
  f_v(x) := \dfrac{1}{2} x^2 - v\epsilon x \,,\quad x \in \cK\,.
  \label{eq:socfdef}
\end{align}
It is easy to see that $\{f_+, f_-\}\subset \F$.
%%%%%%%%%%%%%%%%%%%%%%%%%%%%%%%%%%%%%%%%%%%%%%%%%%%%%%%%%%%%%%%%%%%%%%%%%%%%%%% 

Clearly, $f_v$ is minimized at $x^*_v = v\epsilon$.
By the definition of $f_v$, we have
%Using the fact that $f_+$ and $f_-$ are strongly convex with associated constant $\left(\dfrac{\epsilon}{2}\right)$, we obtain
\begin{align}
  f_v(x) - f_v(x^*_v)
\ge  \dfrac{\epsilon^2}{2}  \indic{x v  < 0}. \label{eq:fv-lb-sc}
\end{align}
We will consider the oracles $\gamma_v$ defined as 
\begin{align}
 \gamma_v(x) = x-v\epsilon + v \min(\epsilon,C_1 \delta^p) + \xi, \label{eq:oracle-1d}
\end{align}
where $\xi \sim \normal(0,\frac{C_2}{\delta^q})$; as with $f_v$, we will also use $\gamma_{+}$ ($\gamma_-$) 
to denote $\gamma_{+1}$ (resp., $\gamma_{-1}$).
The oracle is indeed a $(c_1,c_2)$ type-I oracle, with $c_1(\delta)=C_1\delta^p$ and $c_2(\delta)=\frac{C_2}{\delta^q}$.
%%%%%%%%%%%%%%%%%%%%%%%%%%%%%%%%%%%%%%%%%%%%%%%%%%%%%%%%%%%%%%%%%%%%%%%%%%%%%%% 

Using arguments similar to those in the proof of lower bound for convex functions, we obtain
\begin{align}
\Delta_n^{(1)*}:=\Delta_n^{*} %\nonumber\\
  \ge  &\inf_{\A} \dfrac{\epsilon^2 }{2}  \,\left( 1 - \left(\frac12\dkl{P_{+}}{P_{-}}\right)^{\frac{1}{2}}\right), \label{eq:pinskersc}
\end{align}
Note that $P_+$ (resp. $P_-$) is $\P$ conditioned on the event $V=+1$ (resp. $V=-1$).
%%%%%%%%%%%%%%%%%%%%%%%%%%%%%%%%%%%%%%%%%%%%%%%%%%%%%%%%%%%%%%%%%%%%%%%%%%%%%%% 

Observe that, for any $x\in \R$, $f_-'(x) - f_+'(x) = 2\epsilon$ and hence
\begin{align}
 |\gamma_+(x) - \gamma_-(x)| 
& = | f'_+(x) - \min(\epsilon,C_1 \delta^p) - (f'_-(x)+\min(\epsilon,C_1 \delta^p)) | 
 = 2 (\epsilon - C_1 \delta^p)^+.
 \label{eq:gdiff-ub-sc}
\end{align}
From the foregoing, 
\begin{align}
 \MoveEqLeft \dkl{P_{+}^t(\cdot\mid g_{1:t-1})}{P_{-}^t(\cdot\mid g_{1:t-1})}
 \le  \dfrac{2\{(\epsilon-C_1\delta_t^p)^+\}^2\delta_t^q}{{C_2}},\label{eq:dkgausssc}
\end{align}
where the inequality \eqref{eq:dkgausssc} follows from \eqref{eq:gdiff-ub-sc}.
Thus, we obtain
\begin{align}
\dkl{P_{+}}{P_{-}} \le 2n \sup_{\delta>0} \dfrac{\{(\epsilon-C_1\delta^p)^+\}^2 \delta^q}{C_2}.
\label{eq:kluboundsc}
\end{align}
Substituting the above bound into \eqref{eq:pinskersc}, we obtain 
\begin{align}
 \Delta_n^{(1)*}
  \ge & \dfrac{\epsilon^2}{2} \left(1 - \sqrt{
    n}  \sup_{\delta>0}\dfrac{(\epsilon-C_1\delta^p)^+\delta^{q/2}}{\sqrt{C_2}}
  \right)\,.\label{eq:final-lower-bd-sc}
\end{align}

\paragraph{Derivation of the rates uniformly for all $\delta$:}
As in the proof of the lower bound for $\F_{L,0}(\K)$, we replace the positive part function in \eqref{eq:final-lower-bd-sc} and optimize over $\delta$ to obtain
that the right-hand side of~\eqref{eq:kluboundsc} is optimized by
\begin{align}
\delta_*=\left(\frac{ \epsilon q}{C_1(2p+q)}\right)^{1/p}\,.
\label{eq:deltastar-sc}
\end{align}

From the above, we have 
\[
\Delta_n^{(1)*} \ge \dfrac{\epsilon^2}{2} \left(1 - \sqrt{
    n}  \dfrac{ (\epsilon-C_1\delta_*^p)\delta_*^{q/2}}{\sqrt{C_2}}
  \right)= \dfrac{\epsilon^2}{2} \left(1 - \sqrt{n}  K_1 \epsilon^{\frac{p+\tfrac{q}{2}}{p}}\right)\,, 
\]
 where $K_1 = \dfrac{p}{\sqrt{C_2}(p+\tfrac{q}{2})} \left(\dfrac{q}{2C_1(p+\tfrac{q}{2})}\right)^{\frac{q}{2p}}$.

Plugging in $\epsilon = \left(\dfrac{4p}{(6p+q)\sqrt{n} K_1} \right)^{\frac{2p}{2p+q}}$, we obtain
\begin{align}
\Delta_n^{(1)*} \ge 2^{\frac{2p-q}{2p+q}} \frac{(2p+q)^3}{q^{\frac{2q}{2p+q}}(6p+q)^{\frac{6p+q}{2p+q}}}  C_1^{\frac{2q}{2p+q}}C_2^{\frac{2p}{2p+q}} n^{-\frac{2p}{2p+q}}.
\label{eq:lb-pq-sc}
\end{align}

Now, when $q=2$ and $p=1$, the lower bound in \eqref{eq:lb-pq-sc} simplifies to
\[
\Delta_n^{(1)*} \ge \frac{1}{2} C_1 C_2^{1/2} n^{-1/2}.
\]
On the other hand, for $p=q=2$, we obtain
\[
\Delta_n^{(1)*} \ge 27\left(\frac{2}{7^7}\right)^{\frac{1}{3}} C_1^{2/3}C_2^{2/3} n^{-2/3} .
\]

%%%%%%%%%%%%%%%%%%%%%%%%%%%%%%%%%%%%%%%%%%%%%%%%%%%%%%%%%%%%%%%%%%%%%%%%%%%%%%%
%%%%%%%%%%%%%%%%%%%%%%%%%%%%%%%%%%%%%%%%%%%%%%%%%%%%%%%%%%%%%%%%%%%%%%%%%%%%%%%

\paragraph{Generalization to $d$ dimensions:}
Recall that in this result, $\norm{\cdot} = \norm{\cdot}_2$.
The proof in $d$ dimensions for strongly convex functions is the same as that for the case of smooth convex functions
with the difference that we use~\eqref{eq:socfdef} in defining the functions $f^{(i)}_{v_i}$.
Then, for any $v\in \{\pm 1 \}^d$, $f_v\in \F_{L,1}(\cK)$. Indeed, $f_v(x) = \sum_{i=1}^d f^{(i)}(x_i)$,
hence $\nabla^2 f_v(x) =  I_{d\times d}$, where $I_{d\times d}$ is the $d\times d$ identity matrix.
Thus, $\lambda_{\min}(\nabla^2 f_v(x)) = \lambda_{\max}(\nabla^2 f_v(x)) = 1$.
From~\eqref{eq:dlb} and~\eqref{eq:lb-pq-sc} we get
\begin{align}
\Delta_n^{(d)*} \ge \Delta_n^{(1)*} \,.
\label{eq:lb-pq-d}
\end{align}
\end{proof}

\section{Gradient Estimation Proofs}
\label{sec:appendix-grad}
In this section we present the proofs corresponding to the oracles introduced in \cref{sec:sbco}.

\subsection{Proof of \cref{prop:grad-onepoint}}

\textbf{Case 1 ($f \in \C^3$): }\ \\
We use the proof technique of \cite{spall1997one}.
We start by bounding the bias.
Since by assumption $\EE{ \xi|V}=0$, we have
\begin{align*}
\E\left[  V\left(\dfrac{\xi}{\delta}\right) \right]= 0\,,
\end{align*}
implying that
\begin{align*}
\E[G] =  \E\left[ V \left(\dfrac{f(x+\delta U) }{\delta}\right)\right] \,.
\end{align*}
By Taylor's theorem, we obtain, a.s.,
\begin{align*}
f(x + \delta U) =
 f(x)
 +\delta\,  U\tr\,\nabla f(x)
  + \frac{\delta^2}{2}\, U\tr \nabla^2 f(x) U
  +  \frac{\delta^3}{2} \, R^{+}(x,\delta,U) \,(U, U, U),
\end{align*}
where
\begin{align}
 R^{+}(x,\delta,U)= \int_0^1  \nabla^3 f(  x + s \, \delta U ) (1-s)^2 ds. \label{eq:taylor-r}
\end{align}
In the above, $\nabla^3 f(\cdot)$ is considered as a rank-3 tensor.
Letting $B_3 = \sup_{x\in D} \norm{ \nabla^3 f(x) }$,%
\footnote{Here, $\norm{\cdot}$ is the implied norm: For a rank-3 tensor $T$, $\norm{T} = \sup_{x,y,z\ne 0}
\frac{|T (x,y,z)|}{\norm{x}\norm{y}\norm{z}}$.
}
we have $\norm{ R^{+}(x,\delta,U)} \le B_3/3$ a.s.
Now,
\begin{align*}
\EE{V\, \dfrac{f(x+\delta U)}{\delta}}
&= \EE{V \frac{f(x)}{\delta}} +  \EE{VU^{\tr}
\, \nabla f(x)}  + \EE{\frac{\delta}{2}\, V U\tr \nabla^2 f(x) U} \\
&\qquad+   \EE{\frac{\delta^2}{2}  V \,R^{+}(x,\delta,U)(U \otimes U \otimes U)}
\\
&= \, \nabla f(x)  + \EE{\frac{\delta^2}{2}  V \,R^{+}(x,\delta,U)(U \otimes U \otimes U)}\,.
\end{align*}
The final equality above follows from the facts that $\EE{V} = 0$, $\EE{V U\tr} = I$ and for any $i,j=1,\ldots,d$, $E[V_i U_j^2] = 0$ since $V$ is a deterministic odd function of $U$, with $U$ having a symmetric distribution.
Using the fact that $|R^{+}(x,\delta,U) (U \otimes U \otimes U)| \le
\norm{R^{+}(x,\delta,U)} \norm{U}^3$,
we obtain
\begin{align*}
\norm{ \EE{ G } - \nabla f(x) }_*
\le C_1\,\, \delta^2 \,,
\end{align*}
where $C_1 = \frac{B_3 \EE{ \norm{V}_* \norm{U}^3 }}{6}$.

Let us now bound the variance of $G$:
Using the identity $\E\left\|X -  E[X]\right\|^2 \le 4 \E \left\|X\right\|^2$, which holds for any random variable $X$,%
\footnote{When $\norm{\cdot}$ is defined from an inner product,
$\E\left\|X -  E[X]\right\|^2 = \EE{\norm{X}^2} - \norm{\EE{X}}^2 \le \EE{\norm{X}^2}$ also holds, shaving off a factor of four from the inequality below.}
we bound $\E\left\| G - \E G\right\|_*^2$ as follows:
\begin{align}
\E\left\| G - \E G\right\|_*^2
 &\le 4 \E \left\|G\right\|_*^2 \nonumber \\
& =  4\E\left( \left\| V \right\|_*^2 \left(\left(\dfrac{\xi}{\delta}\right)^2  + 2 \left(\dfrac{\xi}{\delta}\right) \left(\dfrac{f(x+\delta U)}{\delta}\right)
+ \left( \dfrac{f(x+\delta U) }{\delta} \right)^2 \right)\right) \nonumber \\
&=  4\E\left( \left\| V \right\|_*^2 \left(\dfrac{\xi}{\delta}\right)^2\right)
+ 4 \E \left(\left\| V \right\|_*^2 \right)\left( \dfrac{f(x+\delta U) }{\delta} \right)^2  \label{eq:h31} \\
& \le  \frac{C_2}{\delta^2}\,, \nonumber 
\end{align}
where $C_2 = 4 \EE{\norm{V}_*^2}\left( \sigma_\xi^2+B_0^2\right)$, where
$\sigma_\xi^2 = \essup \EE{\xi^2|V}$ and $B_0 = \sup_{x\in \D} f(x)$.
The equality in \eqref{eq:h31} follows from $\EE{ \xi \,|\, V } = 0$.

Therefore, for $f \in \C^3$, $\gamma$ defined by \eqref{eq:one-point} is a $(C_1\delta^2, C_2/\delta^2)$ type-I oracle.

\paragraph{Case 2 ($f$ is convex and $L$-smooth):}\ \\
Since $f$ is convex and $L$-smooth, for any $0<\delta <1$,
\begin{align*}
 0 \le \frac{f(x + \delta u)-f(x)}{\delta}-\<\nabla f(x), u\> \le&   \frac{L \delta \norm{ u}^2}{2}.
\end{align*}
Denoting
$\phi(x,\delta,u):=\frac{f(x + \delta u)-f(x)}{\delta}-\<\nabla f(x), u\>$, we have
$\left|\phi(x,\delta,u) \right| \le  \dfrac{L\delta}{2} \norm{u}^2$.
Then,  given $\EE{VU^\top}=I$, $\EE{V}=0$, we obtain
\begin{align}
\norm{ \EE{ G } - \nabla f(x) }_*
&= \norm{ \EE{\frac{f(x + \delta U)}{\delta}V}-\EE{VU^\top \nabla f(x)}  }_*\nonumber\\
&=\norm{ \EE{V\left( \frac{f(x + \delta U)}{\delta}- -U^\top \nabla f(x)\right)}  }_*\nonumber\\
&= \norm{ \EE{V\left( \phi(x,\delta,U)+\dfrac{f(x)}{\delta} \right)}  }_*\nonumber\\
&=\norm{ \EE{V\phi(x,\delta,U)} }_*\nonumber\\
&\le C_1 \,\, \delta\,, \label{eq:c1onepoint}
\end{align}
  where $C_1 = \dfrac{L}{2}\EE{\norm{V}_*\norm{U}^2}$.
The claim regarding the variance of $G$ follows in a similar manner as in Case 1, i.e., $f \in \C^3$.

Therefore, for $f$ convex and $L$-smooth, $\gamma$ defined by \eqref{eq:one-point} is a $(C_1\delta, C_2/\delta^2)$ type-I oracle, where $C_1$ is given by \eqref{eq:c1onepoint} and $C_2$ as defined in Case 1.
%%%%%%%%%%%%%%%%%%%%%%%%%%%%%%%%%%%%%%%%%%%%%%%%%%%%%%%%%%%%%%%%%%%%%%%%%%%%%%%%%%%%%%%%%%%%%%%%%%%%%%%%%%%%%%%%%%%%%%%%%%%%%%%%%%%%%%%%%%%%%%%%%%%%%%%%%%
%%%%%%%%%%%%%%%%%%%%%%%%%%%%%%%%%%%%%%%%%%%%%%%%%%%%%%%%%%%%%%%%%%%%%%%%%%%%%%%%%%%%%%%%%%%%%%%%%%%%%%%%%%%%%%%%%%%%%%%%%%%%%%%%%%%%%%%%%%%%%%%%%%%%%%%%%%
%%%%%%%%%%%%%%%%%%%%%%%%%%%%%%%%%%%%%%%%%%%%%%%%%%%%%%%%%%%%%%%%%%%%%%%%%%%%%%%%%%%%%%%%%%%%%%%%%%%%%%%%%%%%%%%%%%%%%%%%%%%%%%%%%%%%%%%%%%%%%%%%%%%%%%%%%%

\subsection{Proof of \cref{prop:flaxman}}
Before the proof, we introduce a fundamental theorem of vector calculus, which is commonly known as the Gauss-Ostrogradsky theorem or the divergence theorem . A special case of the theorem for real-valued functions in $\R^n$ can be stated as follows.
\begin{lemma}
\label{lem:gradientCalculus}
Suppose $W \subset \R^n$ is an open set with the boundary $\partial W$. At each point of $\partial W$ there is a normal vector $n_W$ such that $n_W$  (i) has unit norm, (ii) is orthogonal to $\partial W$, (iii) points outward from $W$. Suppose $f: \R^n\to \R$ is a function of class $C^1$ defined at least on the closure of $W$, then we have
\begin{align*}
\int_{ W} \nabla f\,d W = \int_{\partial W} f n_W \,d \partial W \,.
\end{align*}
\end{lemma}

\begin{proof}
Given that $\EE{\norm{V}_*^2}$ and $\EE{\xi^2}$ are bounded, the variance of $G$ remains the same as stated in \cref{prop:grad-onepoint}.

As to the bias, let $\tilde{f}$ be a smoothed version of $f$, i.e., $\forall x \in \cK$,
\begin{align*}
\tilde{f}(x) &= \EE{f(x+\delta V)} =\int_{v \in W} f(x+\delta v)\dfrac{\,d w}{\lvert W\rvert}\,,
\end{align*}
where the expectation is w.r.t. $V$, which is a random variable uniformly chosen from $W$. The second equality interprets the expectation as integral.
Now we want to prove that for any given $x\in \cK$, $G$ is an unbiased gradient estimate of $\tilde{f}$ at $x$.
Since $U$ is uniformly distributed over $\partial W$,  the expectation of $G$ can be written as
\begin{align*}
\EE{G} =\dfrac{\lvert \partial W\rvert}{\lvert W \rvert} \int_{\partial W} \dfrac{1}{\delta} f(x+\delta U)n_W(U)\dfrac{\,d U}{\lvert \partial W\rvert}
=  \int_W \nabla f(x+\delta U)\dfrac{\,d U}{\lvert W\rvert}\,,
\end{align*}
where the second equality follows from \cref{lem:gradientCalculus}, by replacing the gradient of $\hat{f}(u) =\dfrac{1}{\delta} f(x+\delta u) $ with $\nabla f(x+\delta u)$.
Then, the order of the gradient and the integral can be exchanged, because $\int_W f(x+\delta U)\,d U$ exists. Consequently, we obtain $\EE{G} = \nabla \tilde{f}(x)$.

Moreover, $\tilde{f}$ and $f$ are actually close. In particular, for any $x \in \cK$,
\begin{align}
\label{eq:f2tildef}
\tilde{f}(x)-f(x) =\int_{ W} f(x+\delta w)-f(x)\dfrac{\,d w}{\lvert W\rvert} \,.
\end{align}

When $f$ is $L_0$-Lipschitz, $\vert f(x+\delta w)-f(x)\vert \le L_0\delta \norm{w}$, which combined with \eqref{eq:f2tildef} gives that $\gamma$ is a type-II oracle with $c_1(\delta) = C_1 \delta$, where $C_1 = L_0 \sup_{w\in W}\norm{w}$.

When $f$ is convex and $L$-smooth,  $0\le f(x+\delta w)-f(x) - \ip{\nabla f(x), \delta w}\le\dfrac{L}{2}\delta^2 \norm{w}^2$. Given that $W$ is symmetric, $\int_{ W} \ip{\nabla f(x), \delta w} \,d w=0$. Hence, one can easily get that $\gamma$ is a type-II oracle with $c_1(\delta)=C'_1 \delta^2$, where $C'_1 = \dfrac{L}{2 |W|}\int_{ W}\norm{w}^2\,d w$.

Finally, if $f$ is $L$-smooth,
\begin{align*}
\norm{\nabla \tilde{f}(x)- \nabla f(x)}_*
&\le \int_W \norm{\nabla f(x+\delta w) - \nabla f(x)}_* \frac{dw}{|W|}
\le L\delta^2\int_W \norm{w}^2 \frac{dw}{|W|} =2 C'_1 \delta^2
\end{align*}
with the same value of $C'_1$ as before. So $\gamma$ is also a  type-I oracle with $c_1(\delta)=2C'_1 \delta^2$.
\end{proof}

\subsection{Proof of \cref{prop:grad-spsa}}
\textbf{Case 1 ($f \in \C^3$):}\ \\
We use the proof technique of \cite{spall1992multivariate}
  (in particular, Lemma 1 there).
We start by bounding the bias.
Since by assumption $\EE{ \xi^+-\xi^-|V}=0$, we have
\begin{align*}
\E\left[  V\left(\dfrac{\xi_n^+ - \xi_n^-}{2\delta}\right) \right]= 0\,,
\end{align*}
implying that
\begin{align*}
\E[G] =  \E\left[V\,  \dfrac{f(X^+)  -f(X^-)}{2\delta} \right]\,.
\end{align*}

By Taylor's theorem, using that $f\in C^3$, we obtain, a.s.,
\begin{align*}
f(x \pm \delta U) =
 f(x)
 \pm\delta\,  U\tr\,\nabla f(x)
  + \frac{\delta^2}{2}\, U\tr \nabla^2 f(x) U
  \pm  \frac{\delta^3}{2} \, R^{\pm}(x,\delta,U) \,(U, U, U),
\end{align*}
where, as in the proof of \cref{prop:grad-onepoint}, $R^{\pm}(x,\delta,U)$ is defined as follows:
\begin{align}
 R^{\pm}(x,\delta,U)= \int_0^1  \nabla^3 f(  x \pm s \, \delta U ) (1-s)^2 ds. \label{eq:taylor-r-1p}
\end{align}
Letting $B_3 = \sup_{x\in D} \norm{ \nabla^3 f(x) }$,%
we have $\norm{ R^{\pm}(x,\delta,U)} \le B_3/3$ a.s.
Now,
\begin{align}
\begin{split}
\MoveEqLeft       V\, \dfrac{f(X^+)-f(X^-)}{2\delta}
  = V\, \dfrac{f(x+\delta U) - f(x-\delta U)}{2\delta} \\
&= VU^{\tr}
\, \nabla f(x)   +   \frac{\delta^2}{4}  V \,(R^{+}(x,\delta,U)+R^{-}(x,\delta,U))(U \otimes U \otimes U)\,.
\end{split}
\label{eq:l1}
\end{align}
and therefore,
by taking expectations of both sides,
using $\EE{V U\tr} = I$ and then $|R^{\pm}(x,\delta,U) (U \otimes U \otimes U)| \le
\norm{R^{\pm}(x,\delta,U)} \norm{U}^3$,
we get that
\begin{align*}
\norm{ \EE{ G } - \nabla f(x) }_*
\le C_1\,\, \delta^2 \,,
\end{align*}
where $C_1 = \frac{B_3 \EE{ \norm{V}_* \norm{U}^3 }}{6}$.

Using arguments similar to that in the proof of \cref{prop:grad-onepoint}, the variance of $G$ is bounded as follows:
\begin{align}
\MoveEqLeft \E\left\| G - \E G\right\|_*^2
 \le 4 \E \left\|G\right\|_*^2 \nonumber \\
& =  4\E\left( \left\| V \right\|_*^2 \left(\left(\dfrac{\xi^+ - \xi^-}{2\delta}\right)^2  + 2 \left(\dfrac{\xi^+ - \xi^-}{2\delta}\right) \left(\dfrac{f(X^+) - f(X^-)}{2\delta}\right)
+ \left( \dfrac{f(X^+) - f(X^-)}{2\delta} \right)^2 \right)\right) \nonumber \\
&=  4\E\left( \left\| V \right\|_*^2 \left(\dfrac{\xi^+ - \xi^-}{2\delta}\right)^2\right)
+ 4 \E \left(\left\| V \right\|_*^2 \right)\left( \dfrac{f(X^+) - f(X^-)}{2\delta} \right)^2  \label{eq:h3} \\
& \le  \frac{C_2}{\delta^2}\,, \nonumber \label{eq:h4}
\end{align}
where $C_2 = 4 \EE{\norm{V}_*^2}\left( \sigma_\xi^2+\fspan(f)\right)$
and $\fspan(f) = \sup_{x\in \D} f(x) - \inf_{x\in \D} f(x)$.
The equality in \eqref{eq:h3} follows from $\EE{ \xi^+-\xi^- \,|\, U,V } = 0$.

Therefore, for $f \in \C^3$, $\gamma$ defined by \eqref{eq:twosp} is a $(C_1\delta^2, C_2/\delta^2)$ type-I oracle.

\paragraph{Case 2 (Controlled noise and $F$ is convex and $L_{\psi}$-smooth):}\ \\
The proof follows by parallel arguments to that used in the proof of Lemma 1 in \cite{duchi2015optimal} and we give it here for the sake of completeness.

For any convex function $f$ with an $L$-Lipschitz gradient, for any $\delta>0$ it holds that
\begin{align*}
\frac{\<\nabla f(x), \delta u\>}{2\delta} \le \frac{f(x + \delta u) -  f(x)}{2\delta} \le& \frac{\<\nabla f(x), \delta u\> + (L / 2) \norm{\delta u}^2}{2\delta}.
\end{align*}
Using similar inequalities for $f(x-\delta u)$, we obtain
\begin{align*}
\<\nabla f(x), u\> - \frac{L \delta \norm{ u}^2}{2} \le \frac{f(x + \delta u) -  f(x-\delta u)}{2\delta} \le& \<\nabla f(x), u\> + \frac{L \delta \norm{ u}^2}{2}.
\end{align*}
Letting
$\phi(x,\delta,u):=\frac1{\delta}\left(\frac{f(x + \delta u) -  f(x-\delta u)}{2\delta} - \<\nabla f(x),  u\>\right)$, we get
\begin{align*}
\left|\phi(x,\delta,u) \right| \le&  \dfrac{L}{2} \norm{u}^2\,.
% \frac{f(x - \delta u) -  f(x)}{\delta} \le& -\frac{\<\nabla f(x), \delta u\> + (L / 2) \norm{\delta u}^2}{\delta}
\end{align*}

Using $\EE{V U^\top}=I$, we obtain
\begin{align*}
\E\left[V\,  \left(\frac{f(x+\delta U)  -f(x-\delta U)}{2\delta}\right)\right]=&
\E\left[ V U^\top\nabla f(x) +  \delta\phi(x,\delta,U) V \right]\\
= &  \nabla f(x) + \delta \widehat\phi(x,\delta),
\end{align*}
where $\widehat\phi(x,\delta)$ satisfies $\scnorm{\widehat\phi(x,\delta)}_*  \le \, \dfrac{L}{2}\E[ \dnorm{V} \norm{U}^2]$.

Applying the above expression to $F(\cdot, \Psi)$ and recalling that $G=V\,  \left(\tfrac{F(X^+,\psi)  -F(X^-,\psi)}{2\delta}\right)$, we have, for $P$-almost every $\psi$, 
$$\E[G] = \nabla F(x,\psi) + \delta \widehat\phi(x,\delta),$$
where, as before, $\widehat\phi(x,\delta)$ satisfies $\scnorm{\widehat\phi(x,\delta)}_*  \le \, \dfrac{L_{\psi}}{2}\E[ \dnorm{V} \norm{U}^2]$.

Using the fact that $E[\nabla F(x,\Psi)] = \nabla f(x)$, we obtain
\begin{align*}
 \norm{\E[G] - \nabla f(x)}_*
 &= \scnorm{\E\left[V\,  \left(\frac{f(x+\delta U)  -f(x-\delta U)}{2\delta}\right)-V U^\top\nabla f(x) \right]}_*
 \le \, \delta \norm{\E[ V \phi(x,\delta, U)]}_*\\
 &\le \,\frac{\delta \overline{L}_{\Psi}}{2} \E[ \dnorm{V} \norm{U}^2],
\end{align*}
and the claim for the bias follows by setting $C_1= \frac{\overline{L}_{\Psi}}{2} \E[ \dnorm{V} \norm{U}^2]$.

We now bound $\EE{ \norm{G}_*^2}$ as follows:
\begin{align*}
 \E \norm{G}^2
& = \mathbb{E}\norm{V\left(\delta \phi(x,\delta, U)+ U^\top\nabla f(x) \right)}^2
 \le  \E\left[ \left( \dnorm{ V U \tr \nabla f(x)} + \frac{\delta L}{2} \dnorm{V} \norm{U}^2 \right)^2\right]\\
& \le  2 \E\left[  \dnorm{ V U \tr \nabla f(x)}^2\right]  + \frac{\delta^2 \overline{L}_{\Psi}^2}{2}\E\left[ \dnorm{V}^2 \norm{U}^4 \right],
\end{align*}
and the claim for the variance follows by setting $C_2 =  2 B_1^2  + \frac{ \overline{L}_{\Psi}^2}{2}\E\left[ \dnorm{V}^2 \norm{U}^4 \right]$ with $B_1 = \sup_{x\in \K} \scnorm{\nabla f(x)}_*$.

Therefore, for the case of controlled noise with a convex and $L_{\psi}$-smooth $F$, we have that $\gamma$ defined by \eqref{eq:twosp} is a $(C_1\delta, C_2)$ type-I oracle.

\section{Conclusions}
\label{sec:conc}
We presented a novel noisy  gradient oracle model for convex optimization. The oracle model covers several gradient estimation methods in the literature designed for algorithms that can observe only noisy function values, while allowing to handle explicitly the bias-variance tradeoff of these estimators. The framework allows to derive sharp upper and lower bounds on the minimax optimization error and the regret in the online case. From our lower bounds it follows that the current state of the art in designing and analyzing noisy gradient methods for stochastic and online smooth bandit convex optimization are suboptimal.

\appendix
\section{Proof of \cref{lem:ub}}
\label{sec:lemub-proof}
\label{sec:ublemma-proof}
Before the proof, we introduce a well-known bound on the instantaneous linearized "forward-peeking" regret of Mirror Descent.
\begin{lemma}
\label{lem:mdlinregret}
For any $x \in \cK$ and any $t \ge 1$,
\begin{align*}
\MoveEqLeft
\ip{G_t, X_{t+1}-x}
\le \dfrac{1}{\eta_t} \left( \DR(x,X_t)-\DR(x,X_{t+1})-\DR(X_{t+1},X_t) \right)\,,
\end{align*}
where $X_{t+1}$ is selected as in \cref{alg}.
\end{lemma}
\begin{proof}
The point $X_{t+1}$ is the minimizer of
$\Psi_{t+1}(x)=\eta_t \ip{G_t,x}+\DR(x,X_t)$ over $\cK$. Since the gradient of $\Psi_{t+1}(x)$ is
\[
\nabla \Psi_{t+1}(x) = \eta_t G_t + \nabla\mathcal{R}(x)-\nabla\mathcal{R}(X_t),
\]
by the optimality condition, for any $x \in \cK$,
\[
\ip{\eta_t G_t + \nabla\mathcal{R}(x)-\nabla\mathcal{R}(X_t), x-X_{t+1}}\ge 0 \,,
\]
which is equivalent to the result by substituting the definition of the Bregman divergence $\DR$.
\end{proof}

With this, we can turn to the proof of \cref{lem:ub}.

\begin{proof}
From the smoothness and convexity of $f$, and using the strong convexity of $\cR$, we get
\begin{align}
\lefteqn{f(X_{t+1}) - f(x)  } \nonumber \\
& \le  f(X_t) + \ip{\nabla f(X_t), X_{t+1}-X_t} + \frac{L}{2} \norm{X_{t+1}-X_t}^2 - \big\{f(X_t) + \ip{\nabla f(X_t) , x- X_t} \big\}  \nonumber \\
& = \ip{\nabla f(X_t), X_{t+1}-x} + \frac{L}{2}\norm{ X_{t+1} - X_t}^2 \nonumber \\
& \le \ip{\nabla f(X_t), X_{t+1}-x} + \frac{L}{\alpha} D_{\cR}(X_{t+1},X_t)~. \label{eq:mdsmoothnoisy1}
\end{align}
Writing
$\nabla f (X_t) = (\nabla f(X_t)-\overline G_t)  + \xi_t + G_t$ where $\xi_t = \overline G_t - G_t$ is the ``noise'', and using the Cauchy-Schwartz inequality
and the strong convexity of $\cR$,
we obtain
\begin{align*}
\ip{\nabla f(X_{t}),X_{t+1}-x}
 &= \ip{(\nabla f(X_t)-\overline G_t)  + \xi_t + G_t,X_{t+1}-x}  \nonumber \\
 & \le \norm{X_t-x}\norm{\nabla f-\overline G_t}_* + \ip{\xi_t,X_{t+1}-x} + \ip{ G_t, X_{t+1} - x } \nonumber \\
%&\le \beta_t \norm{X_{t+1}-x} + \\
&\le \beta_t \sqrt{ \frac{2D}{\alpha} } + \ip{\xi_t,X_{t+1}-x} + \ip{ G_t, X_{t+1} - x }~.
\end{align*}
After plugging this into~\eqref{eq:mdsmoothnoisy1},
the plan is to take the conditional expectation of both sides w.r.t.\  $\cS_{t}$.
As $X_t$ is $\cS_{t}$-measurable and $\EE{ \xi_t|\cS_{t}} = 0$ by the definition of $\xi_t$ and $\overline G_t$,
we have
\begin{align*}
\MoveEqLeft
\EE{\ip{\xi_t, X_{t+1}-x}|\cS_{t}} = \underbrace{\EE{\ip{\xi_t,X_{t}-x}|\cS_{t}}}_{=0} + \EE{\ip{\xi_t,X_{t+1}-X_t}|\cS_{t}} \,.
\end{align*}
The second term inside the expectation can be bounded by the Fenchel-Young inequality and the strong convexity of $\cR$ as
\[
\ip{\xi_t,X_{t+1}-X_t} \le \frac12 \left(\frac{\norm{\xi_t}_*^2}{a_t} + a_t \norm{X_{t+1}-X_t}^2\right)\,
\le \frac12 \left(\frac{\norm{\xi_t}_*^2}{a_t} + \frac{2a_t}{\alpha} D_{\mathcal{R}}(X_{t+1},X_t) \right)\,.
\]
Applying
\cref{lem:mdlinregret}
to bound $\ip{ G_t, X_{t+1} - x }$, and putting everything together gives
\begin{align*}
 \EE{f(X_{t+1}) - f(x) |\cS_{t} }
\le & \quad
 \beta_t \sqrt{ \frac{2D}{\alpha} }+
\frac{1}{2a_t}  \EE{\norm{\xi_t}_*^2|\cS_{t}}+
\frac{1}{\eta_t} \left(D_{\mathcal{R}}(x,X_t)-D_{\mathcal{R}}(x,X_{t+1})\right) \\&+
\underbrace{\left(\frac{a_t+L}{\alpha}-\frac{1}{\eta_t}\right) D_{\mathcal{R}}(X_{t+1},X_t)}_{=0}\numberthis \label{eq:ubProofWithNoise} \,.
\end{align*}
Finally, we sum up these inequalities for $t=1,\dots,n-1$. Since the divergence terms telescope, recall \eqref{eq:div-telescope},
by the tower rule and using $\sigma_t^2 = \EE{ \norm{\xi_t}_*^2}$, we obtain
\begin{align*}
 \EE{ \sum_{t=1}^n f(X_t) - f(x) }
& \le
  \EE{f(X_1)-f(x)} + \sqrt{\tfrac{2D}{\alpha}} \sum_{t=1}^{n-1} \beta_t +
	   \frac{D}{\eta_{n-1}} +
	  \sum_{t=1}^{n-1}\frac{\sigma_t^2}{2a_t}\\
&=
  \EE{f(X_1)-f(x)} + \sqrt{\tfrac{2D}{\alpha}} \sum_{t=1}^{n-1} \beta_t +
	   \frac{D(a_{n-1}+L)}{\alpha} +
	  \sum_{t=1}^{n-1}\frac{\sigma_t^2}{2a_t}\,.
\end{align*}

When $f$ is $L$-smooth and $\mu$-strongly convex,  we can rewrite \eqref{eq:mdsmoothnoisy1} as
\begin{align*}
\MoveEqLeft
f(X_{t+1}) - f(x) \nonumber\\
 &\le f(X_t) + \ip{ \hat{G}_t, X_{t+1}-X_t } + \frac{L}{2} \norm{X_{t+1}-X_t}^2 - \left\{ f(X_t) + \ip{ \hat{G}_t, x-X_t}+\dfrac{\mu}{2}\DR(x, X_t) \right\} \nonumber \\
 &= \ip{\hat{G}_t, X_{t+1} - x } +  \frac{L}{2} \norm{X_{t+1}-X_t}^2-\dfrac{\mu}{2}\DR(x, X_t) \nonumber\\
 &\le \ip{\hat{G}_t,X_{t+1}-x} + \frac{L}{\alpha} D_{\mathcal{R}}(X_{t+1},X_t)-\dfrac{\mu}{2}\DR(x, X_t)\,.
\end{align*}
Now, we  obtain the following along the lines of \eqref{eq:ubProofWithNoise}:
\begin{align*}
&\EE{f(X_{t+1}) - f(x) |\cS_{t} }
 \le
 \beta_t \sqrt{ \frac{2D}{\alpha} }+
\frac{1}{2a_t}  \EE{\norm{\xi_t}_*^2|\cS_{t}}\\
&\qquad\qquad\qquad\quad +\left(\dfrac{1}{\eta_t}-\dfrac{\mu}{2}  \right)\DR(x,X_t)-\dfrac{1}{\eta_t}\DR(x,X_{t+1})
+\left( \dfrac{L+a_t}{\alpha}-\dfrac{1}{\eta_t} \right)\DR(X_{t+1},X_t) \,.
\end{align*}
Since $\dfrac{1}{\eta_t}=\dfrac{\mu t}{2}=\dfrac{L+a_t}{\alpha}$ by definition, summing up theses inequalities for $t=1,2,\ldots,n-1$, we obtain
\begin{align*}
 \EE{ \sum_{t=1}^n f(X_t) - f(x) }
& \le
  \EE{f(X_1)-f(x)} + \sqrt{\tfrac{2D}{\alpha}} \sum_{t=1}^{n-1} \beta_t +
	  \sum_{t=1}^{n-1}\frac{\sigma_t^2}{2a_t}-\dfrac{1}{\eta_{n-1}}\DR(x,X_{n})\\
& \le
  \EE{f(X_1)-f(x)} + \sqrt{\tfrac{2D}{\alpha}} \sum_{t=1}^{n-1} \beta_t +
	  \sum_{t=1}^{n-1}\frac{\sigma_t^2}{2a_t}\,,
\end{align*}
and the claim follows.
\end{proof}

\section*{Acknowledgements}
The authors would like to thank Kai Zheng for his help in pointing out the error in previous work.
This work was supported by the Alberta Innovates Technology Futures through the Alberta Ingenuity Centre for Machine Learning and RLAI, NSERC, the National Science Foundation (NSF) under Grants CMMI-1434419, CNS-1446665, and CMMI-1362303, and by the Air Force Office of Scientific Research (AFOSR) under Grant FA9550-15-10050.

\bibliographystyle{apalike}
\bibliography{main}

\end{document}